\documentclass[11pt]{article}
\usepackage[margin=1in]{geometry}

\usepackage[utf8]{inputenc} 
\usepackage[backref,colorlinks,citecolor=blue,bookmarks=true]{hyperref}       
\usepackage{mathtools, amssymb, amsthm}
\usepackage{dsfont}
\usepackage[capitalize]{cleveref}
\usepackage{autonum}
\usepackage[ruled]{algorithm}
\usepackage{algpseudocode}
\usepackage{enumitem}
\usepackage{bbm}
\usepackage{pgfplots}
\usepackage{subcaption}

\theoremstyle{plain}
\newtheorem{theorem}{Theorem}[section]
\newtheorem{lemma}[theorem]{Lemma}
\newtheorem{corollary}[theorem]{Corollary}

\theoremstyle{definition}
\newtheorem{definition}[theorem]{Definition}

\theoremstyle{remark}
\newtheorem{remark}[theorem]{Remark}

\newcommand*{\N}{{\mathbb{N}}}
\newcommand*{\Z}{{\mathbb{Z}}}
\newcommand*{\R}{{\mathbb{R}}}

\let\eps\epsilon
\DeclareMathOperator*{\pr}{\mathbb{P}}
\DeclareMathOperator*{\ex}{\mathbb{E}}
\DeclareMathOperator*{\var}{Var}

\newcommand{\ind}{\mathbbm{1}}

\DeclarePairedDelimiter{\floor}{\lfloor}{\rfloor}
\DeclarePairedDelimiter{\round}{\lfloor}{\rceil}

\DeclareMathOperator{\sgn}{sign}
\DeclareMathOperator{\thres}{thres}
\DeclareMathOperator{\relu}{ReLU}

\DeclareMathOperator{\App}{App}
\DeclareMathOperator{\bad}{bad}

\DeclareMathOperator{\unif}{Unif}

\let\hat\widehat

\newcommand{\lwr}{\textsf{LWR}}
\newcommand{\lwe}{\textsf{LWE}}
\newcommand{\dv}{\textsf{DV}}

\newcommand{\calN}{\mathcal{N}}
\newcommand{\calC}{\mathcal{C}}

\newcommand{\calA}{\mathcal{A}}
\newcommand{\calB}{\mathcal{B}}
\newcommand{\calX}{\mathcal{X}}
\newcommand{\calY}{\mathcal{Y}}

\DeclareMathOperator{\poly}{poly}

\DeclareMathOperator{\Id}{Id}

\DeclareMathOperator{\negl}{negl}

\title{Hardness of Noise-Free Learning for \\ Two-Hidden-Layer Neural Networks}
\author{Sitan Chen\thanks{\texttt{sitanc@berkeley.edu}. This work was supported in part by NSF Award 2103300.} \\
	UC Berkeley
	\and Aravind Gollakota\thanks{\texttt{aravindg@cs.utexas.edu}. Supported by NSF awards AF-1909204, AF-1717896, and the NSF AI Institute for Foundations of Machine Learning (IFML).} \\ UT Austin
	\and Adam R. Klivans\thanks{\texttt{klivans@cs.utexas.edu}. Supported by NSF awards AF-1909204, AF-1717896, and the NSF AI Institute for Foundations of Machine Learning (IFML).} \\
	UT Austin
	\and Raghu Meka\thanks{\texttt{raghum@cs.ucla.edu}. Supported by NSF CAREER Award CCF-1553605.} \\
	UCLA
}
\date{November 13, 2022}

\DeclarePairedDelimiter{\ceil}{\lceil}{\rceil}

\DeclarePairedDelimiter{\iprod}{\langle}{\rangle}

\DeclarePairedDelimiter{\brc}{\{}{\}}

\newcommand{\cube}[1]{\brc{\pm 1}^{#1}}
\renewcommand{\epsilon}{\varepsilon}
\newcommand{\wt}[1]{\widetilde{#1}}
\newcommand{\liftsym}{\vartriangle} 

\newcommand{\calD}{\mathcal{D}}

\begin{document}
	
	\maketitle
	\begin{abstract}
		We give superpolynomial statistical query (SQ) lower bounds for learning two-hidden-layer ReLU networks with respect to Gaussian inputs in the standard (noise-free) model. No general SQ lower bounds were known for learning ReLU networks of any depth in this setting: previous SQ lower bounds held only for adversarial noise models (agnostic learning) \cite{klivans2014embedding,goel2020statistical,diakonikolas2020near} or restricted models such as correlational SQ \cite{goel2020superpolynomial,diakonikolas2020algorithms}.
		
		Prior work hinted at the impossibility of our result: Vempala and Wilmes \cite{vempala2019gradient} showed that general SQ lower bounds cannot apply to any real-valued family of functions that satisfies a simple non-degeneracy condition. 
		
		To circumvent their result, we refine a lifting procedure due to Daniely and Vardi \cite{daniely2021local} that reduces Boolean PAC learning problems to Gaussian ones.  We show how to extend their technique to other learning models and, in many well-studied cases, obtain a more efficient reduction. As such, we also prove new cryptographic hardness results for PAC learning two-hidden-layer ReLU networks, as well as new lower bounds for learning constant-depth ReLU networks from label queries.
	\end{abstract}
	\section{Introduction}
	
	
	In this paper we extend a central line of research proving representation-independent hardness results for learning classes of neural networks. We will consider arguably the simplest possible setting: given samples $(x_1,y_1),\ldots,(x_n,y_n)$ where for every $i\in[n]$, $x_i$ is sampled independently from some distribution $\calD$ over $\R^d$ and $y_i = f(x_i)$ for an unknown neural network $f: \R^d\to\R$, the goal is to output any function $\hat{f}$ for which $\ex_{x\sim\calD}[(f(x) - \hat{f}(x))^2]$ is small. This model is often referred to as the \emph{realizable} or \emph{noise-free} setting. 
	
	This problem has long been known to be computationally hard for discrete input distributions. For example, if $\calD$ is supported over a discrete domain like the Boolean hypercube, then we have a variety of hardness results based on cryptographic/average-case assumptions \cite{klivans2009cryptographic, daniely2014average, daniely2016complexity, daniely2020hardness, daniely2021local}.
	
	Over the last few years there has been a very active line of research on the complexity of learning with respect to continuous distributions, the most widely studied case being the assumption that $\calD$ is a standard Gaussian in $d$ dimensions. A rich algorithmic toolbox has been developed for the Gaussian setting \cite{janzamin2015beating,zhong2017recovery,brutzkus2017globally,LiY17,Tian17,convotron,ge2018learning,bakshi2019learning,zgu,diakonikolas2020approximation,LiMZ20,diakonikolas2021small,awasthi2021efficient,chen2020learning,song2021cryptographic,vodrahalli2022algorithms}, but all known efficient algorithms can only handle networks with a single hidden layer, that is, functions of the form $f(x) = \sum^k_{i=1} \lambda_i \sigma(\iprod{w_i,x})$.
	This motivates the following well-studied question:
	\begin{equation}
		\text{\emph{Are there fundamental barriers to learning neural networks with} two \emph{hidden layers?}} \label{eq:question}
	\end{equation}
	
	Two distinct lines of research, one using cryptography and one using the statistical query (SQ) model, have made progress towards solving this question.
	
	In the cryptographic setting, \cite{daniely2021local} showed that the existence of a certain class of pseudorandom generators, specifically local pseudorandom generators with polynomial stretch, implies superpolynomial lower bounds for learning ReLU networks with \emph{three} hidden layers.
	
	
	For SQ learning, work of \cite{goel2020superpolynomial} and \cite{diakonikolas2020algorithms} gave the first superpolynomial {\em correlational} SQ (CSQ) lower bounds for learning even one-hidden-layer neural networks. 
	Notably, however, there are strong separations between SQ and CSQ \cite{andoni2014learning,andoni2019attribute,chen2020learning}, and the question of whether a general SQ algorithm exists remained an interesting open problem.
	In fact, Vempala and Wilmes \cite{vempala2019gradient} showed that general SQ lower bounds might be impossible to achieve for learning real-valued neural networks. For any family of networks satisfying a simple non-degeneracy condition (see Section~\ref{sec:discussion}), they gave an algorithm that succeeded using only polynomially many statistical queries.
	As such, the prevailing conventional wisdom was that noise was required in the model to obtain full SQ lower bounds.  
	
	The main contribution of this paper is to answer Question \ref{eq:question} by giving both general SQ lower bounds and cryptographic hardness results (based on the Learning with Rounding or $\lwr$ assumption) for learning ReLU networks with two hidden layers and polynomially bounded weights.\footnote{Note that if the weights were allowed to be arbitrarily large, it is well-known to be trivial to obtain hardness over Gaussian inputs from hardness over Boolean inputs: simply approximate the sign function arbitrarily well and convert all but an arbitrarily small fraction of Gaussian inputs to bitstrings.} We note that our SQ lower bound is the first of its kind for learning ReLU networks of {\em any} depth. We also show how to extend our results to the setting where the learner has label query access to the unknown network.
	
	\begin{table}[h]
	\setlength{\tabcolsep}{0.5em}
	\renewcommand{\arraystretch}{1.2}
	\centering
    \begin{tabular}{ c||c|c } 
     Reference & Num.\ hidden layers & Model of hardness \\ 
     \hline\hline
     \protect\cite{diakonikolas2020algorithms,goel2020superpolynomial} & 1 & Correlational SQ \\
     \hline
     \protect\cite{daniely2021local} & 3 & \begin{tabular}{@{}c@{}}Cryptographic \\ (assuming existence of local PRGs)\end{tabular} \\
     \hline
     \emph{This work} & 2 & Full SQ \\
     \hline
     \emph{This work} & 2 & \begin{tabular}{@{}c@{}}Cryptographic \\ (assuming hardness of LWR)\end{tabular} \\
    \end{tabular}
    \caption{Summary of known and new superpolynomial lower bounds for learning noise-free shallow ReLU networks over Gaussian inputs up to sufficiently small (but non-negligible) error. (Definitions and terminology may be found in \cref{sec:prelims}.)}
    \label{tab:summary}
    \end{table}
	
	\paragraph{SQ Lower Bound} We state an informal version of our main SQ lower bound:
	
	\begin{theorem}[Full SQ lower bound for two hidden layers (informal), see \cref{thm:depth2-sq-lb}] \label{thm:main_informal}
		Any SQ algorithm for learning $\poly(d)$-sized two-hidden-layer ReLU networks over $\calN(0,\Id_d)$ up to squared loss $1/\poly(d)$ must use at least $d^{\omega(1)}$ queries, or have query tolerance that is negligible in $d$.
	\end{theorem}
	We stress that this bound holds unconditionally, independent of any cryptographic assumptions. This simultaneously closes the gap between the hardness result of \cite{daniely2021local} and the positive results on one-hidden-layer networks \cite{janzamin2015beating,zhong2017recovery,ge2018learning,awasthi2021efficient,diakonikolas2021small} and goes against the conventional wisdom that one cannot hope to prove full SQ lower bounds for learning real-valued functions in the realizable setting.
	
	We also note that unlike previous CSQ lower bounds which are based on orthogonal function families and crucially exploit cancellations specific to the Gaussian distribution, our Theorem~\ref{thm:main_informal} and other hardness results in this paper easily extend to any reasonably anticoncentrated and symmetric product distribution over $\R^d$; see Remark~\ref{rem:distribution}.
	
	\paragraph{Cryptographic Lower Bound} While Theorem~\ref{thm:main_informal} rules out almost all known approaches for provably learning neural networks (e.g. method of moments/tensor decomposition \cite{janzamin2015beating,zhong2017recovery,ge2018learning,bakshi2019learning,diakonikolas2020approximation,diakonikolas2021small,awasthi2021efficient}, noisy gradient descent \cite{brutzkus2017globally,LiY17,Tian17,convotron,zgu,LiMZ20}, and filtered PCA \cite{chen2020learning}), it does not preclude the existence of a non-SQ algorithm for doing so. Indeed, a number of recent works \cite{bruna2021continuous,song2021cryptographic,zadik2022latticebased,diakonikolas2021nongaussian} have ported algorithmic techniques like lattice basis reduction \cite{lenstra1982factoring}, traditionally studied in the context discrete settings like cryptanalysis, to learning problems over continuous domains for which there is no corresponding SQ algorithm.
	
	Our next result shows however that under a certain cryptographic assumption, namely hardness of \emph{Learning with Rounding ($\lwr$) with polynomial modulus} \cite{banerjee2012pseudorandom,alwen2013learning,bogdanov2016hardness} (see Section~\ref{sec:prelims}), no polynomial-time algorithm can learn two-hidden-layer neural networks from Gaussian examples.
	
	\begin{theorem}[Cryptographic hardness result (informal), see \cref{thm:depth2-crypto-hardness}]\label{thm:crypto_informal}
		Suppose there exists a $\poly(d)$-time algorithm for learning $\poly(d)$-sized two-hidden-layer ReLU networks over $\calN(0, \Id_d)$ up to squared loss $1/\poly(d)$. Then there exists a quasipolynomial-time algorithm for $\lwr$ with polynomial modulus.
	\end{theorem}
	
	Note that here we may actually improve the $\lwr$ hardness assumption required from quasipolynomial to any mildly superpolynomial function of the security parameter (see \cref{rem:crypto-quasipoly}).
	
	Under $\lwr$ with polynomial modulus, we also show the first hardness result for learning \emph{one hidden layer} ReLU networks over the uniform distribution on $\{0,1\}^d$ (see Theorem~\ref{thm:uniform}).
	
	In Section~\ref{sec:prelims}, we discuss existing hardness evidence for $\lwr$ as well as its relation to more standard assumptions like Learning with Errors. From a negative perspective, Theorem~\ref{thm:crypto_informal} suggests that the aforementioned lattice-based algorithms for continuous domains are unlikely to yield new learning algorithms for two-hidden-layer networks, because even their more widely studied \emph{discrete} counterparts have yet to break $\lwr$. From a positive perspective, in light of the prominent role $\lwr$ and its variants have played in a number of practical proposals for post-quantum cryptography \cite{cheon2018lizard,bhattacharya2018round5,jin2016optimal,d2018saber}, Theorem~\ref{thm:crypto_informal} offers a new avenue for stress-testing these schemes.
	
	\paragraph{Query Learning Lower Bound}
	
	One additional benefit of our techniques is that they are flexible enough to accommodate other learning models beyond traditional PAC learning. To illustrate this, for our final result we show hardness of learning neural networks from \emph{label queries}. In this setting, the learner is much more powerful: rather than sample or SQ access, they are given the ability to query the value $f(x)$ of the unknown function $f$ at any desired point $x$ in $\R^d$, and the goal is still to output a function $\hat{f}$ for which $\ex[(f(x) - \hat{f}(x))^2]$ is small. The expectation here is with respect to some specified distribution, which we will take to be $\calN(0,\Id_d)$, though as before, our techniques will apply to any reasonably anticoncentrated, symmetric product distribution over $\R^d$.
	
	In recent years, this question has received renewed interest from the security and privacy communities in light of \emph{model extraction attacks}, which attempt to reverse-engineer neural networks found in publicly deployed systems \cite{TramerZJRR16,MilliSDH19,Papernot,JagielskiCBKP20,RolnickK20,JWZ20,daniely2021exact}. Recent work \cite{chen2021efficiently} has shown that in this model, there is an efficient algorithm for learning arbitrary one-hidden-layer ReLU networks that is truly polynomial in all relevant parameters. We show that under plausible cryptographic assumptions about the existence of simple pseudorandom function (PRF) families (see Section~\ref{sec:query}) which may themselves be based on standard number theoretic or lattice-based cryptographic assumptions, such a guarantee is impossible for general \emph{constant-depth} ReLU networks.
	
	\begin{theorem}[Label query hardness (informal), see \cref{thm:model_hardness}]\label{thm:model_informal}
		If either the decisional Diffie--Hellman or the Learning with Errors assumption holds, then the class of $\poly(d)$-sized constant-depth ReLU networks from $\R^d$ to $\R$ is not learnable up to small constant squared loss $\epsilon$ over $\calN(0,\Id_d)$ even using label queries over all of $\R^d$.
	\end{theorem}
	
	Note that the connection between PRFs and hardness of learning from label queries over \emph{discrete domains} is a well-known connection dating back to Valiant~\cite{valiant1984theory}. To our knowledge, however, Theorem~\ref{thm:model_informal} is the first hardness result for query learning over continuous domains.

	\subsection{Discussion and Related Work}
	\label{sec:discussion}
	
	\paragraph{Hardness for learning neural networks.} There are a number of works \cite{blum1989training,vu2006infeasibility,klivans2009cryptographic,livni2014computational,goel2017reliably,daniely2020hardness} showing hardness for \emph{distribution-free} learning of various classes of neural networks. 
	
	As for hardness of distribution-specific learning, several works have established lower bounds with respect to the Gaussian distribution. Apart from the works \cite{goel2020superpolynomial,diakonikolas2020algorithms,daniely2021local} from the introduction which are most closely related to the present work, we also mention the works of
	\cite{klivans2014embedding,goel2019time,goel2020statistical,diakonikolas2020near} which showed hardness for \emph{agnostically} learning halfspaces and ReLUs, \cite{shamir2018distribution} which showed hardness for learning periodic activations with gradient-based methods, \cite{song2017complexity} which showed lower bounds against SQ algorithms for learning one-hidden-layer networks using \emph{Lipschitz} statistical queries and large tolerance, and \cite{song2021cryptographic} which showed lattice-based hardness of learning one-hidden-layer networks when the labels $y_i$ have been perturbed by bounded \emph{adversarially chosen} noise. Our approach has similarities to the ``Gaussian lift'' as studied by Klivans and Kothari~\cite{klivans2014embedding}. Their approach, however, required noise in the labels, whereas we are interested in hardness in the strictly \emph{realizable setting}.
	We also remark that \cite{das2020learnability,agarwal2021deep} showed \emph{correlational} SQ lower bounds for learning random depth-$\omega(\log n)$ neural networks over \emph{Boolean inputs} which are uniform over a halfspace.
	
	There have also been works on hardness of learning from label queries over \emph{discrete domains} and for more ``classical'' concept classes like Boolean circuits \cite{feldman2009power,cohen2015aggregate,valiant1984theory,kharitonov1995cryptographic,angluin1995won}.
	
	Lastly, we remark on how our results relate to \cite{chen2020learning}, which gives the only known upper bound for learning neural networks over Gaussian inputs beyond one hidden layer. They showed that learning ReLU networks of arbitrary depth is ``fixed-parameter tractable'' in the sense that there is a fixed function $g(k,\epsilon)$ in the size $k$ of the network and target error $\epsilon$ for which the time complexity is at most $g(k,\epsilon)\cdot \poly(d)$, and their algorithm can be implemented in SQ. That said, this does not contradict our lower bounds for two reasons: 1) their algorithm only applies to networks without biases, 2) in our lower bound constructions, $k$ scales polynomially in $d$.
	
	\paragraph{SQ lower bounds for real-valued functions.} A recurring conundrum in the literature on SQ lower bounds for supervised learning has been whether one can show SQ hardness for learning \emph{real-valued} functions. SQ lower bounds for Boolean functions are typically shown by lower bounding the \emph{statistical dimension} of the function class, which essentially corresponds to the largest possible set of functions in the class which are all approximately pairwise orthogonal. Indeed, the content of the hardness results of \cite{goel2020superpolynomial,diakonikolas2020algorithms} was to prove lower bounds on the statistical dimension of one-hidden-layer networks. Unfortunately, for real-valued functions, statistical dimension lower bounds only imply CSQ lower bounds. As discussed in \cite{goel2020superpolynomial}, the class of $d$-variate Hermite polynomials of degree-$\ell$ is pairwise orthogonal and of size $d^{O(\ell)}$, which translates to a CSQ lower bound of $d^{\Omega(\ell)}$. Yet there exist SQ algorithms for learning Hermite polynomials in far fewer queries \cite{andoni2014learning,andoni2019attribute}.
	
	Further justification for the difficulty of proving SQ lower bounds for real-valued functions came from \cite{vempala2019gradient}, which observed that for any real-valued learning problem satisfying a seemingly innocuous non-degeneracy assumption---namely that for any pair of functions $f,g$ in the class, the probability under the input distribution $\calD$ that $f(x) = g(x)$ is zero---there is an efficient ``cheating'' SQ algorithm (see Proposition 4.1 therein). The SQ lower bound shown in the present work circumvents this proof barrier by exhibiting a family of neural networks for which any pair of networks agrees on a set of inputs with Gaussian measure \emph{bounded away from zero}.
	
	\paragraph{Open questions}
	
	While our results settle Question~\ref{eq:question}, a number of intriguing gaps between our lower bounds and existing upper bounds remain open:
	\begin{itemize}[leftmargin=*]
		\item \textbf{General one-hidden-layer networks.} Despite the considerable amount of work on learning one-hidden-layer networks over Gaussian inputs, all known positive results that run in polynomial time in all parameters (input dimension $d$, network size $k$, inverse error $1/\epsilon$) still need to make various assumptions on the structure of the network. Remarkably, it is even open whether one-hidden-layer ReLU networks with \emph{positive output layer weights} (i.e. ``sums of ReLUs'') can be learned in polynomial time, the best known guarantee being the $(k/\epsilon)^{\log^2 k}\cdot \poly(d/\epsilon)$-time algorithm of \cite{diakonikolas2021small}. As for general one-hidden-layer ReLU networks, it is still open whether they can even be learned in time $d^{O(k)}\cdot\text{poly}(1/\epsilon)$, the best known guarantee being the $k^{\poly(k/\epsilon)}\cdot \poly(d)$-time algorithm of \cite{chen2020learning}.
		\item \textbf{Query learning shallow networks.} While Theorem~\ref{thm:model_informal} establishes that above a certain constant depth, ReLU networks cannot be learned even from label queries over the Gaussian distribution. It would be interesting to close the gap between this and the positive result of \cite{chen2021efficiently} which only applies to one-hidden-layer networks, although fully settling this seems closely related to the question of what are the shallowest possible Boolean circuits needed to implement pseudorandom functions, a longstanding open question in circuit complexity.
	\end{itemize}

	\subsection{Technical Overview}\label{sec:overview}
	
	Our work will build on a recent approach of Daniely and Vardi \cite{daniely2021local}, who developed a simple and clever technique for lifting discrete functions to the Gaussian domain entirely in the realizable setting. Our main contributions are to (1) make their lifting procedure more efficient so that two hidden layers suffice and (2) show how to apply the lift in a variety of models beyond PAC. For the purposes of this overview we will take the domain of our discrete functions to be $\{0,1\}^d$, but our techniques extend to $\Z_q^d$ with $q = \poly(d)$.
	
	\paragraph{Daniely--Vardi (DV) lift.}
	At a high level, the DV lift is a transformation mapping a Boolean example $(x,y)$ labeled by a hard-to-learn Boolean function $f$ to a Gaussian example $(z,\wt{y})$ labeled by a (real-valued) ReLU network $f^{\dv}$ that behaves similarly to $f$ in that $f^{\dv}(z)$ approximates $f(\sgn(z))$, where for us $\sgn(t)$ denotes $\ind[t > 0]$ and is applied elementwise. The key idea is to use a continuous approximation $\wt{\sgn}$ of the $\sgn$ function, and to pair it with a ``soft indicator'' function $\bad : \R^d \to \R_+$ that is large whenever $\sgn(z) \neq \wt{\sgn}(z)$, and that can be implemented as a one-hidden-layer network independent of the target function. One can show that whenever $f$ is realizable as an $L$-hidden-layer network over $\{0,1\}^d$, the function $f^{\dv}(z) = \relu(f(\wt{\sgn}(z)) - \bad(z))$ can be implemented as an $(L+2)$-hidden-layer network satisfying \begin{equation}
		f^{\dv}(z) = \relu(f(\sgn(z)) - \bad(z)). \label{eq:dvform_overview}
	\end{equation}
	This property allows us to generate synthetic Gaussian labeled examples $(z, f^{\dv}(z))$ from Boolean labeled examples $(x, f(x))$, and thereby reduce the problem of learning $f$ to that of learning $f^{\dv}$. For a fuller overview, see \cref{subsec:dv_overview}.

	\paragraph{Improving the DV lift.} 
	Our first technical contribution is to introduce a more efficient lift which only requires one extra hidden layer. Our starting point is to observe that a variety of hard-to-learn Boolean functions $f$ like parity and $\lwr$ take the form $f(x) = \sigma(h(x))$ for some ReLU network $h$ whose range $T$ over Boolean inputs is a discrete subset of $[0,\poly(d)]$ of \emph{polynomially bounded} size, and for some function $\sigma: T\to[0,1]$. For such \emph{compressible} functions (see \cref{def:compressible}), one can write $f(x) = \sigma(h(x)) = \sum_{t^* \in T} \sigma(t^*)\ind[h(x) = t^*]$. Again, we would like to implement lifted function $f^{\liftsym} : \R^d \to \R$ using $\wt{\sgn}$ and $\bad$ so that it approximates $f(\sgn(z))$ except when $\bad$ indicates that $\wt{\sgn} \neq \sgn$. To this end, we might hope to implement, say, \[ f^{\liftsym}(z) = \sum_{t^* \in T} \sigma(t^*) \ind[h(\wt{\sgn}(z)) = t^*] \ind[\forall j : \bad(z_j) \ll 1]. \] Here we now view $\bad$ as a univariate function, and whenever it is small, we can be sure $\wt{\sgn} = \sgn$. Suppose that we could build a one-hidden-layer network $N(s_1, \dots, s_d; t)$ that behaves like $\ind[t = 0]\ind[\forall j : s_j \ll 1]$. Then we could realize $f^{\liftsym}$ as an $(L+1)$-hidden-layer network: \begin{equation} f^{\liftsym}(z) = \sum_{t^* \in T} \sigma(t^*) N(\bad(z_1), \dots, \bad(z_d);\ h(\wt{\sgn}(z)) - t^*). \label{eq:dvform_comp}
	\end{equation} While many natural attempts to build such an $N$ run into difficulties, we construct a suitably relaxed version of $N$ that turns out to suffice for the reduction. To gain some intuition for our construction, the starting observation is that the following inclusion-exclusion type formula vanishes identically whenever any of the $s_j$ is $1$: \begin{align}
	    &\psi(s_1, s_2, s_3) - \psi(1, s_2, s_3) - \psi(s_1, 1, s_3) - \psi(s_1, s_2, 1) \\
		&+ \psi(s_1, 1, 1) + \psi(1, s_2, 1) + \psi(s_1, 1, 1) - \psi(1, 1, 1).
	\end{align} For a suitable choice of $\psi$, one might hope to build $N$ out of such a formula by taking $s_j = \bad(z_j)$ for every $j$. But the natural generalization of this expression to $d$ inputs would have size $2^d$, which runs the risk of rendering the resulting SQ lower bounds vacuous. Our final construction (Lemma~\ref{lem:n3}) instead resembles a truncated inclusion-exclusion type formula of only quasipolynomial size, which may be of independent interest. Since the SQ lower bounds for Boolean functions that we build on are exponential, by a simple padding argument we still obtain a superpolynomial SQ lower bound for our lifted functions.
	
	\paragraph{Hard one-hidden-layer Boolean functions and \lwr.} To use this lift for Theorems~\ref{thm:main_informal} and \ref{thm:crypto_informal}, we need one-hidden-layer networks that are \emph{compressible} and hard to learn over uniform Boolean inputs. For SQ lower bounds, we can simply start from parities, for which there are exponential SQ lower bounds, and which turn out to be easily implementable by compressible one-hidden-layer networks.
	For cryptographic hardness, Daniely and Vardi \cite{daniely2021local} used certain one-hidden-layer Boolean networks that arise from the cryptographic assumption that local PRGs exist (see Section A.4.1 therein). Unfortunately, these functions are not compressible. For this reason, we work instead with $\lwr$: it turns out that the $\lwr$ functions are compressible and, conveniently, the hardness assumption directly involves uniform discrete inputs.

	\paragraph{Hardness beyond PAC.} While the DV lift is \emph{a priori} only for showing hardness of example-based PAC learning, we can extend it to the SQ and label query models by simple simulation arguments.

	\section{Preliminaries}
	\label{sec:prelims}
	
	\subsection{Notation}
	We use $\unif(S)$ to denote the uniform distribution over a set $S$. We use $U_d$ as shorthand for $\unif\{0,1\}^d$. We use $\calN(0, \Id_d)$ (or sometimes $\calN_d$ for short) to denote the standard Gaussian, and $|\calN(0, \Id_d)|$ (or $|\calN_d|$ for short) to denote the positive standard half-Gaussian (i.e., $g \sim |\calN(0, \Id_d)|$ if $g = |z|$ for $z \sim \calN(0, \Id_d)$). We use $[n]$ to denote $\{1, \dots, n\}$.
	
	For $q > 0$, $\Z_q$ will denote the integers modulo $q$, which we will identify with $\{0, \dots, q-1\}$. We use $\Z_q / q$ to denote $\{0, 1/q, \dots, (q-1)/q\}$. Our discrete functions will in general have domain $\Z_q^d$ for some $q$. The $q = 2$ case, namely Boolean functions, have domain $\{0,1\}^d$. For the purposes of this paper, $\sgn : \R \to \{0,1\}$ is defined as $\sgn(t) = \ind[t > 0]$. We will extend this to $\Z_q$ by defining $\thres_q : \R \to \Z_q$ in terms of a certain partition of $\R$ into $q$ intervals $I_0, \dots, I_{q-1}$ (formally defined later) as the piecewise constant function that takes on value $k$ on $I_k$ for each $k \in \Z_q$. Scalar functions and scalar arithmetic applied to vectors act elementwise. We say a quantity is \emph{negligible} in a parameter $n$, denoted $\negl(n)$, if it decays as $1/n^{\omega(1)}$.
	
	A \emph{one-hidden-layer} ReLU network mapping $\R^d$ to $\R$ is a linear combination of ReLUs, that is, a function of the form
	\begin{equation}
		F(x) = W_1\relu \big( W_0 x + b_0 \big) + b_1,
	\end{equation}
	where $W_0 \in \R^{k\times d}$, $W_1\in\R^{1\times k}$, $b_0\in\R^k$, and $b_1\in\R$.
	A \emph{two-hidden-layer} ReLU network mapping $\R^d$ to $\R$ is a linear combination of ReLUs of one-hidden-layer networks, that is, a function of the form
	\begin{equation}
		F(x) = W_2 \relu \big(W_1 \relu \big(W_0 x + b_0 \big) + b_1 \big) + b_2,
	\end{equation}
	where $W_0\in\R^{k_0\times d}$, $W_1\in\R^{k_1\times k_0}$, $W_2\in\R^{1\times k_1}$, $b_0\in\R^{k_0}$, $b_1 \in \R^{k_1}$, and $b_2\in\R$. Our usage of the term \emph{hidden layer} thus corresponds to a \emph{nonlinear} layer. 
	
	\subsection{Learning models}
	Let $\calC$ be a function class mapping $\R^d$ to $\R$, and let $\calD$ be a distribution on $\R^d$. We consider various learning models where the learner is given access in different ways to labeled data $(x, f(x))$ for an unknown $f \in \calC$ and must output a (possibly randomized) predictor that achieves (say) squared loss $\epsilon$ for any desired $\epsilon > 0$. In the traditional PAC model, access to the data is in the form of iid labeled examples $(x, f(x))$ where $x \sim \calD$, and the learner is considered efficient if it succeeds using $\poly(d, 1/\eps)$ time and sample complexity. In the Statistical Query (SQ) model \cite{kearns1998efficient,reyzin2020statistical}, access to the data is through an SQ oracle. Given a bounded query $\phi : \R^d \times \R \to [-1,1]$ and a tolerance $\tau > 0$, the oracle may respond with any value $v$ such that $|v - \ex_{x \sim \calD}[\phi(x, f(x))]| \leq \tau$. A correlational query is one that is linear in $y$, i.e.\ of the form $\phi(x,y) = \wt{\phi}(x)y$ for some $\wt{\phi}$, and a correlational SQ (CSQ) learner is one that is only allowed to make CSQs. An SQ learner is considered efficient if it succeeds using $\poly(d, 1/\eps)$ queries and tolerance $\tau \geq 1/\poly(d, 1/\eps)$. Finally, in the label query model, the learner is allowed to request the value of $f(x)$ for any desired $x$, and is considered efficient if it succeeds using $\poly(d, 1/\eps)$ time and queries.
	
	\subsection{Learning with Rounding}\label{subsec:lwr}
	The Learning with Rounding ($\lwr$) problem \cite{banerjee2012pseudorandom} is a close cousin of the well-known Learning with Errors ($\lwe$) problem \cite{regev2009lattices}, except with deterministic rounding in place of random additive errors.
	
	\begin{definition}
		Fix moduli $p, q \in \N$, where $p < q$, and let $n$ be the security parameter. For any $w \in \Z_q^n$, define $f_w : \Z_q^n \to \Z_p / p$ by \[ f_w(x) = \frac{1}{p} \round{w \cdot x}_p = \frac{1}{p} \round{\frac{p}{q} (w \cdot x \bmod q)}, \] where $\round{t}$ is the closest integer to $t$. In the $\lwr_{n, p, q}$ problem, the secret $w$ is drawn randomly from $\Z_q^n$, and we must distinguish between labeled examples $(x, y)$ where $x \sim \Z_q^n$ and either $y = f_w(x)$ or $y$ is drawn independently from $\unif(\Z_p / p)$. The $\lwe_{n,q,B}$ problem is similar, except that $y \in \Z_q / q$ is either $\frac{1}{q}((w \cdot x + e) \bmod q)$ for some $e \in \Z_q$ sampled from a carefully chosen distribution, e.g. discrete Gaussian, such that $|e| \leq B$ except with $\negl(n)$ probability, or is drawn from $\unif(\Z_q / q)$.
	\end{definition}
	
	\begin{remark}
		Traditionally the $\lwr$ problem is stated with labels lying in $\Z_p$ instead of $\Z_p / p$, although both are equivalent since the moduli $p, q$ may be assumed to be known to the learner. The choice of $\Z_p / p$ is simply a convenient way to normalize labels to lie in $[0, 1]$. For consistency, we similarly normalize $\lwe$ labels to lie in $\Z_q / q$.
	\end{remark}
	
	It is known that $\lwe_{n,q,B}$ is as hard as worst-case lattice problems when $q=\poly(n)$ and $B = q/\poly(n)$ (see e.g.\ \cite{regev2010learning,peikert2016decade} for surveys). Yet this is not known to directly imply the hardness of $\lwr_{n,p,q}$ in the regime in which $p, q$ are both $\poly(n)$, which is the one we will be interested in as $p,q$ will dictate the size of the hard networks that we construct in the proof of our cryptographic lower bound. 
	
	Unfortunately, in this polynomial modulus regime, it is only known how to reduce from $\lwe$ to $\lwr$ \emph{when the number of samples is bounded relative to the modulus} \cite{alwen2013learning,bogdanov2016hardness}. For instance, the best known reduction in this regime obtains the following hardness guarantee:
	
	\begin{theorem}[\cite{bogdanov2016hardness}]\label{thm:lwr-lwe-hardness}
		Let $n$ be the security parameter, let $p, q \geq 1$ be moduli, and let $m, B \geq 0$.  Assuming $q \geq \Omega(m B p)$, any distinguisher capable of solving $\lwr_{n,p,q}$ using $m$ samples implies an efficient algorithm for $\lwe_{n,q,B}$.
	\end{theorem}
	For our purposes, Theorem~\ref{thm:lwr-lwe-hardness} is not enough to let us base our Theorem~\ref{thm:crypto_informal} off of $\lwe$, as we are interested in the regime where the learner has an \emph{arbitrary} polynomial number of samples. 
	
	$\lwr$ with polynomial modulus and arbitrary polynomial samples is nevertheless conjectured to be as hard as worst-case lattice problems \cite{banerjee2012pseudorandom} and has already formed the basis for a number of post-quantum cryptographic proposals \cite{d2018saber,cheon2018lizard,bhattacharya2018round5,jin2016optimal}. We remark that one piece of evidence in favor of this conjecture is a reduction from a less standard variant of $\lwe$ in which the usual discrete Gaussian errors are replaced by errors uniformly sampled from the integers $\{-q/2p,\ldots,q/2p\}$ \cite{bogdanov2016hardness}.
	
	Note also that for our purposes we require \emph{quasipolynomial}-time hardness (or $T(n)$-hardness for $T(n)$ being any other fixed, mildly superpolynomial function of the security parameter) of $\lwr$. While slightly stronger than standard polynomial-time hardness, this remains a reasonable assumption since algorithms for worst-case lattice problems are still believed to require at least subexponential time.
	
	\subsection{Partial assignments}
	Let $\alpha \in \brc{0,1,\star}^d$ be a \emph{partial assignment}. We refer to $S(\alpha): \brc{i\in[d]: \alpha_i = \star}\subseteq[d]$ as the set of \emph{free variables} and $[d]\backslash S(\alpha)$ as the set of \emph{fixed variables}. Given two partial assignments $\alpha,\beta$, let the \emph{resolution} $\alpha\searrow\beta$ denote the partial assignment $\gamma$ obtained by substituting $\alpha$ into $\beta$. That is,
	\begin{equation}
		\gamma_i = \begin{cases}
			\star & i\in S(\alpha) \cap S(\beta) \\
			\beta_i & i\in[d]\backslash S(\beta) \\
			\alpha_i & i\in S(\beta)\backslash S(\alpha)
		\end{cases}
	\end{equation}
	
	In this case we say that $\gamma$ is a \emph{refinement} of $\beta$ that is the \emph{result of applying $\alpha$}. We write $\gamma\in\App(\alpha)$ to denote that $\gamma$ is a result of applying $\alpha$. Note that the set of refinements of $\beta$ consists of all $3^{|S(\beta)|}$ partial assignments $\gamma\in\brc{0,1,\star}^d$ which agree with $\beta$ on all fixed variables of $\beta$.
	
	Given $\alpha$, let $w(\alpha)$ denote $|\brc{i: \alpha_i = 1}|$, that is, the Hamming weight of its fixed variables. Note that $w(\alpha\searrow\beta) \le w(\alpha) + w(\beta)$.
	
	Given a function $h:\R^d\to\R$ and partial assignment $\gamma$, we use $h_{\gamma}:\R^d\to\R$ to denote its partial restriction given by substituting in $\gamma_i$ into the $i$-th input coordinate if $\gamma_i\in\brc{0,1}$. Note that given two partial restrictions $\alpha,\beta$,
	\begin{equation}
		(h_\beta)_\alpha = h_{\alpha\searrow\beta} \label{eq:chain}
	\end{equation}
	We say that $\alpha$ is \emph{sorted} if the restriction of $\alpha$ to its fixed variables is sorted in nonincreasing order, e.g. $\alpha = (1,\star,1,\star,\star,0,0)$ is sorted, but $\alpha = (1,\star,0,\star,\star,0,1)$ is not. Given $\alpha$ which is not necessarily sorted, denote its \emph{sorting} by $\overline{\alpha}$. In general, we will use overline notation to denote sorted partial assignments.
	
	\section{Compressing the Daniely--Vardi Lift}\label{sec:lift}
	
	In this section we show how to refine the lifting procedure of Daniely and Vardy~\cite{daniely2021local} such that whenever the underlying discrete functions satisfy a property we term \emph{compressibility}, we obtain hardness under the Gaussian for networks with just one extra hidden layer.
	
	\begin{definition}\label{def:compressible}
		Let $q > 0$ be a modulus.\footnote{Our results are stronger when $q$ is taken to be a large polynomial in the dimension, but the Boolean $q = 2$ case is illustrative of all the main ideas.} We call an $L$-hidden-layer ReLU network $f:\Z_q^d \to [0,1]$ \emph{compressible} if it is expressible in the form $f(x) = \sigma(h(x))$, where \begin{itemize}[nosep]
			\item $h : \Z_q^d \to T$ is an $(L-1)$-hidden-layer network such that $|h(x)| \leq \poly(d)$ for all $x$;
			\item $h$ has range $T = h(\Z_q^d)$ such that $T \subseteq \Z$ and $|T| \leq \poly(d)$; and
			\item $\sigma : T \to [0,1]$ is a mapping from $h$'s possible output values to $[0,1]$.
		\end{itemize} 
	\end{definition}
	\begin{remark}
		To see why such an $f$ is an $L$-hidden-layer  network in $z$, consider the function $\sigma : T \to \R$. Because $T \subseteq \Z$ and $|T| \leq \poly(d)$, $\sigma$ is expressible as (the restriction to $T$ of) a piecewise linear function on $\R$ whose size and maximum slope are $\poly(d)$, and hence as a $\poly(d)$-sized one-hidden-layer  ReLU network from $\R$ to $\R$. By composition, $x \mapsto \sigma(h(x))$ can be represented by an $L$-hidden-layer  network.
	\end{remark}
	
	We now formally state a theorem which captures our ``compressed'' version of the DV lift. The version of this theorem for $L+2$ layers is implicit in \cite{daniely2021local}. In technical terms, our improvement consists of removing the single outer ReLU present in their construction. Thus, while our construction still has three \emph{linear} layers, it has only two \emph{non-linear} layers.
	
	\begin{theorem}[Compressed DV lift]
		\label{thm:continuous-lift}
		Let $q = \poly(d)$ be a modulus. Let $\calC$ be a class of compressible $L$-hidden-layer $\poly(d)$-sized ReLU networks mapping $\Z_q^d$ to $[0, 1]$. Let $m = m(d) = \omega_d(1)$ be a size parameter that grows slowly with $d$. There exists a class $\calC^{\liftsym}$ of $(L+1)$-hidden-layer $d^{\Theta(m)}$-sized ReLU networks mapping $\R^d$ to $[0,1]$ such that the following holds:
		
		Suppose there is an efficient algorithm $A$ capable of learning $\calC^{\liftsym}$ over $\calN(0, \Id_d)$ up to squared loss $d^{-\Theta(m)}$. Then there is an efficient algorithm $B$ capable of \emph{weakly predicting} $\calC$ over $\unif(\Z_q^d)$ with advantage $d^{-\Theta(m)}$ over guessing the constant $1/2$ in the following sense:  given access to labeled examples $(x, f(x))$ for $x \sim \unif(\Z_q^d)$ and an unknown $f \in \calC$, $B$ satisfies \[ \ex \big[ \big(B(x) - f(x)\big)^2 \big] < \ex \big[ \big(\frac{1}{2} - f(x)\big)^2 \big] - d^{-\Theta(m)}, \] where the probability is taken over both $x$ and the internal randomness of $B$.
		We refer to $\calC^{\liftsym}$ as the \emph{lifted class} corresponding to $\calC$.
	\end{theorem}
	
	By a standard padding argument, we obtain the following corollary which lets us work with polynomial-sized neural networks.
	
	\begin{corollary}[Compressed DV lift with padding]\label{cor:padded-lift}
		Let $q$, $m$ and $d$ be as above, and let $d' = d^{m}$. View $\calC$ and $\calC^{\liftsym}$ as function classes on $\Z_q^{d'}$ and $\R^{d'}$ respectively, defined using only the first $d$ coordinates, so that $\calC^{\liftsym}$ is now a $\poly(d')$-sized class over $\R^{d'}$. Then an algorithm capable of learning $\calC^{\liftsym}$ over $\calN_{d'}$ up to squared loss $1/\poly(d')$ implies a weak predictor for $\calC$ over $\unif(\Z_q^{d'})$ with advantage $1/\poly(d')$.
	\end{corollary}
	
	
	\subsection{The DV Lift}
	\label{subsec:dv_overview}
	
	Before proceeding to the proof of \cref{thm:continuous-lift}, we first outline the idea of the original DV lift in the setting of Boolean functions ($q = 2$). The goal is to approximate any given $f \in \calC$ by a ReLU network $f^{\dv} : \R^d \to \R$ in such a way that $f^{\dv}$ under $\calN_d$ behaves similarly to $f$ under $U_d$. As a first attempt, one might consider the function $f^\star(z) = f(\sgn(z))$ (also studied in \cite{klivans2014embedding}), where recall that $\sgn(t) = \ind[t > 0]$. We could implement the following reduction: given a random example $(x, y)$ where $x \sim U_d$ and $y=f(x)$, draw a fresh half-Gaussian $g \sim |\calN_d|$ and output $((2x-1)g, y)$ (where the arithmetic in defining the vector $(2x-1)g$ is done elementwise). Since $2x-1$ is distributed uniformly over $\cube{d}$, the marginal is exactly $\calN_d$, and the labels are consistent with $f^\star$ since $\sgn((2x-1)g) = x$ and so $f(\sgn((2x-1)g)) = f(x)$. However, the issue is that the $\sgn$ function is discontinuous, and so $f^\star$ is not realizable as a ReLU network.
	
	Daniely and Vardi address this concern by devising a clever construction for $f^{\dv}$ that interpolates between two desiderata:
	\begin{itemize}[nosep]
		\item For all but a small fraction of inputs, an initial layer successfully ``Booleanizes'' the input. In this case, one would like $f^{\dv}(z)$ to simply behave as $f(\sgn(z))$.
		\item For the remaining fraction of inputs, we would ideally like $f^{\dv}$ to output an uninformative value such as zero, but this would violate continuity of $f^{\dv}$.
	\end{itemize}
	The trick is to use a continuous approximation of the sign function, $N_1$, that interpolates linearly between $0$ and $1$ on an interval $[-\delta, \delta]$ (see \cref{fig:sgn}), and to pair it with a ``soft indicator'' function $N_2 : \R \to \R$ for the region where $N_1 \neq \sgn$. Concretely, $N_2(t)$ is constructed as a one-hidden-layer ReLU network that (a) is always nonnegative, (b) equals $0$ when $|t| \geq 2\delta$, and (c) equals $1$ when $|t| \leq \delta$ (see \cref{fig:ind}). Now let $N_2'(z) = \sum_j N_2(z_j)$, and define \begin{equation}\label{eq:og-dv-lift}
		f^{\dv}(z) = \relu(f(N_1(z)) - N_2'(z)).
	\end{equation} One can show that $f^{\dv}$ satisfies $ f^{\dv}(z) = \relu(f(\sgn(z)) - N_2'(z))$, since $N_2'$ ``zeroes out'' $f^{\dv}$ wherever $N_1 \neq \sgn$ for any coordinate. This lets us perform the following reduction: given examples $(x, y)$ where $x \sim U_d$ and $y = f(x)$, draw a fresh $g \sim |\calN_d|$ and output $(z, \wt{y}) = ((2x-1)g, \relu(y - N_2'((2x-1)g)))$. The marginal is again $\calN_d$, and the labels are easily seen to be consistent with $f^{\dv}$. Correctness of the reduction can be established by using Gaussian anticoncentration to argue that $f^{\dv}$ is a good approximation of $f$. Formally, one can prove the following theorem.
	
	\begin{theorem}[Original DV lift, implicit in \cite{daniely2021local}]\label{thm:og-dv}
		Let $\calC$ be a class of $L$-hidden-layer $\poly(d)$-sized ReLU networks mapping $\{0,1\}^d$ to $[0, 1]$. There exists a class $\calC^{\dv}$ of $(L+2)$-hidden-layer $\poly(d)$-sized ReLU networks mapping $\R^d$ to $[0,1]$ such that the following holds. Suppose there is an efficient algorithm $A$ capable of learning $\calC^{\dv}$ over $\calN(0, \Id_d)$ up to squared loss $\frac{1}{64}$. Then there is an efficient algorithm $B$ capable of weakly predicting $\calC$ over $\unif\{0,1\}^d$ with squared loss $\frac{1}{16}$.
	\end{theorem}
	
	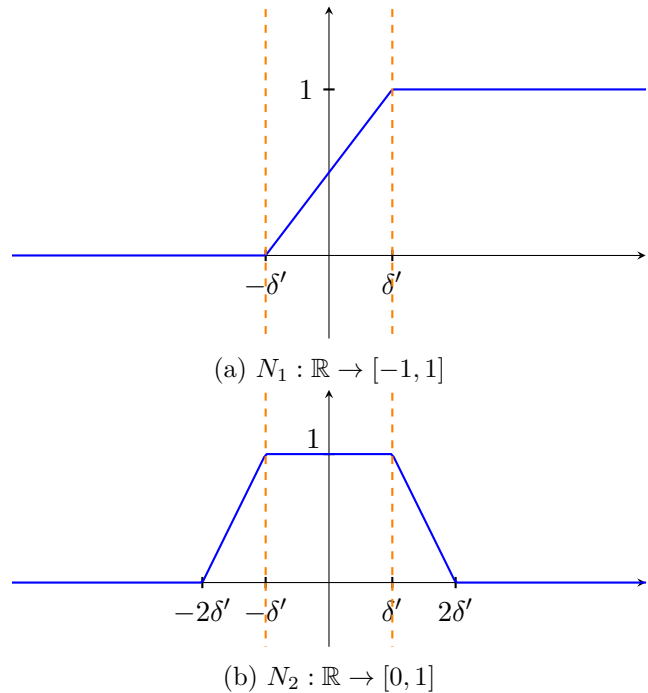
\begin{figure}[H]
		\label{fig:combined}
		\begin{subfigure}{\linewidth}
			\centering
			\begin{tikzpicture}
				\begin{axis}[
					width=10cm, height=6cm,
					axis x line=center, 
					axis y line=middle, 
					samples=500,
					ymin=-0.5, ymax=1.5,
					xmin=-1, xmax=1,
					domain=-2:2,
					xtick={-0.2, 0.2},
					xticklabels={$-\delta'$, $\delta'$},
					ytick={-1, 1},
					every tick/.style={black,thick},
					]
					\addplot [mark=none, thick, blue] {2.5*(max(x + 0.2,0) - max(x-0.2,0))};
					\addplot +[dashed, thick, orange] coordinates {(-.2, -2) (-.2, 2)};
					\addplot +[dashed, thick, orange] coordinates {(.2, -2) (.2, 2)};
				\end{axis}
			\end{tikzpicture}
			\subcaption{$N_1 : \R \to [-1,1]$}\label{fig:sgn}
		\end{subfigure}
		\begin{subfigure}{\linewidth}
			\centering
			\begin{tikzpicture}
				\begin{axis}[
					width=10cm, height=5cm,
					axis x line=center, 
					axis y line=middle, 
					samples=500,
					ymin=-0.5, ymax=1.5,
					xmin=-1, xmax=1,
					domain=-2:2,
					xtick={-0.4, -0.2, 0.2, 0.4},
					xticklabels={$-2\delta'$, $-\delta'$, $\delta'$, $2\delta'$},
					ytick={1},
					yticklabels={$1$},
					yticklabel style={xshift=0.1cm,yshift=0.2cm},
					every tick/.style={black,thick},
					]
					\addplot [mark=none, thick, blue, samples=500] {
						5*(max(x + 0.4, 0) + max(x - 0.4, 0) - max(x + 0.2, 0) - max(x - 0.2, 0))
					};
					\addplot +[dashed, thick, orange] coordinates {(-.2, -5) (-.2, 5)};
					\addplot +[dashed, thick, orange] coordinates {(.2, -5) (.2, 5)};
				\end{axis}
			\end{tikzpicture}
			\subcaption{$N_2 : \R \to [0,1]$}\label{fig:ind}
		\end{subfigure}
		\caption{Schematic plots of $N_1$ and $N_2$ in the $q=2$ case, where $N_2'(z)$ may be realized as $\sum_{j \in [d]} N_2 (z_j)$. Here, $\delta' = \Theta(\delta)$ where $\delta$ is the parameter from \cref{lem:n1,,lem:n2}.} 
	\end{figure}
	
	
	We now show how to construct the gadgets $N_1$ and $N_2$, extending them to make them suitable for working with $\Z_q$ for general $q$ as opposed to just $\{0, 1\}$. These constructions utilize the simple but important property that piecewise linear functions on the real line are readily and efficently realized as linear combination of ReLUs.
	
	Start by letting $I_0, I_1, \dots, I_{q-1}$ be a partition of $\R$ into $q$ consecutive intervals each of mass $1/q$ under $\calN(0,1)$ (e.g., when $q = 2$, $I_0 = (-\infty, 0)$ and $I_1 = (0, \infty)$). Note that these intervals will have differing lengths, and the shortest ones will be the ones closest to the origin. Still, by Gaussian anti-concentration, we know that each $|I_j| \geq \Theta(1/q)$. Let $\thres_q : \R \to \Z_q$ be the piecewise constant function that takes on value $k$ on $I_k$. Clearly, when $t \sim \calN(0,1)$, $\thres_q(t) \sim \unif(\Z_q)$. Let $R_1, \dots, R_{q}$ be intervals such that $R_k \subseteq I_{k-1} \cup I_k$ and $R_k$ contains the boundary point between $I_{k-1}$ and $I_k$, and such that each $R_k$ has mass $\delta / q$ for some $\delta \ll 1$ to be picked later. Let $S_1, \dots, S_q$ be slightly larger intervals such that $R_k \subset S_k$ for each $k \in [q-1]$, and each $S_k$ has mass $2\delta/q$. By Gaussian anti-concentration again, each $|S_k| \geq \Theta(\delta / q)$. Notice that by construction, $\pr_{z \sim \calN(0,1)}[z \in \cup_k R_k] = \delta$ and $\pr_{z \sim \calN(0,1)}[z \in \cup_k S_k] = 2\delta$.
	
	\begin{lemma}\label{lem:n1}
		Let $\delta > 0$, $q > 0$, and intervals $I_k, R_k, S_k$ for $k \in \Z_q$ be as above. There exists a one-hidden-layer ReLU network $N_1 : \R \to \R$ with $O(q)$ units and weights of magnitude $O(q/\delta)$ such that $N_1(t) = \thres_q(t)$ if $t \notin \cup_k R_k$.
	\end{lemma}
	\begin{proof}
		This can be done by considering the piecewise linear function that approximates the function $\thres_q$ by matching it exactly on $\R \setminus \cup_k R_k$, and interpolating  linearly between values $k-1$ and $k$ on the interval $R_k$ for each $k \in [q-1]$.
	\end{proof}
	
	\begin{lemma}\label{lem:n2}
		Let $\delta > 0$, $q > 0$, and intervals $I_k, R_k, S_k$ for $k \in \Z_q$ be as above. There exists a one-hidden-layer ReLU network $N_2 : \R \to [0, 1]$ with $O(q)$ units and weights of magnitude $O(q/\delta)$ such that \[
		N_2(t) \text{ is } \begin{cases}
			= 1 & \text{ if } t \in \cup_k R_k \\
			= 0 & \text{ if } t \in \R \setminus \cup_k S_k \\
			\geq 0 & \text{ otherwise}
		\end{cases} .
		\]
	\end{lemma}
	\begin{proof}
		Consider the piecewise linear function that is $0$ on $\R \setminus \cup_k S_k$, is $1$ on $\cup_k R_k$, and interpolates linearly between $0$ and $1$ (or $1$ and $0$) on $S_k \setminus R_k$ for every $k \in [q-1]$. Put differently, the graph of $N_2$ consists of a trapezoid on each $S_k$ that achieves its maximum value of $1$ on $R_k$.
	\end{proof}
	
	\subsection{Saving One Hidden Layer via Compressibility}\label{subsec:comp-lift}
	
	The starting point for exploiting compressibility to avoid a hidden layer in the lift is as follows. Compressibility lets us express $f(x)$ as $\sigma(h(x))$ for some $h : \Z_q^d \to T$ with a $\poly(d)$-sized range $T \subseteq \Z$, and some $\sigma : T \to [0, 1]$. So we can write \[ f(x) = \sigma(h(x)) = \sum_{t^* \in T} \sigma(t^*)\ind[h(x) = t^*]. \] We would like a lifted function $f^{\liftsym} : \R^d \to \R$ (where we introduce $f^{\liftsym}$ as notation to distinguish our lift from the original DV lift, denoted $f^{\dv}$) such that $f^{\liftsym}(z)$ behaves like $\sigma(h(\thres_q(z)))$ except when $N_2$ indicates that $N_1 \neq \thres_q$, in which case we want $f^{\liftsym}(z) = 0$. To this end, we might hope to write \[ f^{\liftsym}(z) = \sum_{t^* \in T} \sigma(t^*) \ind[h(N_1(z)) = t^*] \ind[\forall j : N_2(z_j) < 1]. \] Suppose that we could build a one-hidden-layer network $N_3(s_1, \dots, s_d; t)$ that behaves like $\ind[t = 0]\ind[\forall j : s_j < 1]$. Then we could realize $f^{\liftsym}$ as \[ f^{\liftsym}(z) = \sum_{t^* \in T} \sigma(t^*) N_3(N_2(z_1), \dots, N_2(z_d);\ h(N_1(z)) - t^*). \] Notice that whenever $N_2(z_j) = 1$ for any coordinate $j$, this expression vanishes. Otherwise, we know that $h(N_1(z)) = h(\thres_q(z))$, which takes values in $T$, so that only the summand with $t^* = h(\thres_q(z))$ survives and the expression simplifies to $f(\thres_q(z)) N_3(N_2(z_1), \dots, N_2(z_d); 0)$. It is not hard to show that this is sufficient to let us complete the required reduction. Moreover, because $N_3$ is a one-hidden-layer network in its arguments, and because both $h \circ N_1$ and $N_2$ have at most $L$ hidden layers (for $h \circ N_1$, one comes from $N_1$ and $L-1$ from $h$; for $N_2$, it itself has just one hidden layer), this implementation of $f^{\liftsym}$ would have only $L+1$ hidden layers.
	
	Slightly more generally, one can show that it would suffice to build a one-hidden-layer network $N_3$ with the following properties: \begin{equation}
		N_3(s_1, \dots, s_d; t) =
		\begin{cases}
			0 & \text{if } \exists j : s_j = 1 \\
			0 & \text{if } t \in \Z \setminus \{0\} \\
			1 & \text{if } \forall j : s_j = 0 \text{ and } t = 0
		\end{cases} \label{eq:ideal}
	\end{equation}
	
	Unfortunately, most natural attempts to construct $N_3$ with such ideal properties --- in particular, all formulations of $N_3$ purely as a function of two variables, $\sum_j s_j$ and $t$, which was the approach taken in \cite{daniely2021local} --- run into difficulties and appear to require \emph{two} hidden layers (see Appendix~\ref{app:barriers} for discussion). One approach that does almost work, however, comes at the cost of exponential size. Let $\psi(s_1, \dots, s_d; t)$ be any function that vanishes whenever $t \in \Z \setminus \{0\}$ (for all $s_1, \dots, s_d \in [0,1]^d$). For simplicity, let us consider the $d = 3$ case. Consider the following expression that resembles the inclusion-exclusion formula: \begin{align}\label{eq:pie-3d}
		&\psi(s_1, s_2, s_3; t) - \psi(1, s_2, s_3; t) - \psi(s_1, 1, s_3; t) - \psi(s_1, s_2, 1; t) \\
		&+ \psi(s_1, 1, 1; t) + \psi(1, s_2, 1; t) + \psi(s_1, 1, 1; t) - \psi(1, 1, 1; t)
	\end{align} Notice that whenever any $s_j = 1$, this expression vanishes identically. Moreover, for any $t \in \Z \setminus \{0\}$ (and any $s_1, \dots, s_d$), the expression vanishes again because each summand vanishes. Thus the first two properties are satisfied; the third property turns out to be more subtle, and we will ignore it for the moment. The natural generalization of this expression to general $d$ can be stated in the language of partial assignments.
	
	\begin{lemma}
		Let $\psi : \R^d \to \R$ be any function. Let $\mathcal{P}_i$ denote the set of partial assignments $\gamma \in \{1,\star\}^d$ with $i$ 1s. The expression \begin{equation}\label{eq:full-pie}
			\sum_{i=0}^{d} \sum_{\gamma \in \mathcal{P}_i} (-1)^{i} \psi_{\gamma}
		\end{equation}
		vanishes whenever any $s_j = 1$. (We may view $t$ as an additional parameter that is always left free, as in \cref{eq:pie-3d})
	\end{lemma}
	\begin{proof}
		For concreteness, suppose $s_1 = 1$. Let $\mathcal{P}^\star_i$ (resp.\ $\mathcal{P}^1_i$) denote the set of $\gamma \in \mathcal{P}_i$ with $s_1 = \star$ (resp.\ $s_1 = 1$). For every $i \in \{0, \dots, d-1\}$, we can form a bijection between $\mathcal{P}^\star_i$ and $\mathcal{P}^1_{i+1}$ using the map $\gamma \mapsto \gamma'$ where $\gamma' = (1, \gamma_2, \dots, \gamma_d)$. When $s_1 = 1$, for every such pair $(\gamma, \gamma')$, we have $\psi_\gamma = \psi_{\gamma'}$, and moreover they occur in \eqref{eq:full-pie} with opposite signs. Thus the entire expression vanishes.
	\end{proof}
	Let us assume for now that $\psi$ is picked suitably and the rest of the reduction goes through with this construction (as one can verify when we come to the proof of \cref{thm:continuous-lift}, this would indeed be the case). This construction has size $2^d$, meaning that the resulting lifted functions would have size $S = \poly(2^d)$.  But by \cref{thm:lwr-sq-lb}, the SQ lower bound for the $\lwr$ functions over $\Z_q^n$ with $n = d$ and $q = \poly(n)$ scales as $q^{\Omega(n)} = 2^{\Omega(d \log d)} = S^{\Omega(\log \log S)}$, which is still superpolynomial in $S$. Thus after padding the dimension to $d' = 2^d$, this construction would actually still yield a superpolynomial SQ lower bound for two-hidden-layer ReLU networks over $\R^{d'}$.
	
	Instead of pursuing this route, however, we give a more efficient construction that has size only slightly superpolynomial in $d$. The key idea is to restrict attention to those possibilities for $(s_1, \dots, s_d) = (N_2(z_1), \dots, N_2(z_d))$ that are the most likely. Specifically, if $m = \omega_d(1)$ is the size parameter from \cref{thm:continuous-lift}, then by setting $\delta$ in \cref{lem:n1,,lem:n2} appropriately, we can ensure that with overwhelming probability over $z \sim \calN(0,\Id)$, no more than $m$ of the $N_2(z_j)$ are simultaneously $1$. Accordingly, we focus on constructing $N_3$ such that \begin{equation}
		N_3(s_1, \dots, s_d; t) =
		\begin{cases}
			0 & \text{if between 1 and $m$ of the $s_i$ are 1} \\
			0 & \text{if } t \in \Z \setminus \{0\} \\
			1 & \text{otherwise}
		\end{cases} .
	\end{equation}
	
	We now describe a $d^{\Theta(m)}$-sized construction for $N_3$ that satisfies the first and second properties exactly, and ``approximately'' satisfies the third in the sense that it takes on a nonzero value with nonnegligible probability over its inputs. As we will see later, this turns out to be enough for the reduction to go through. The construction retains the spirit of using a linear combination of partial restrictions.
	
	\begin{lemma}[Main lemma]\label{lem:main_truncation}
		Let $m = m(d) = \omega_d(1)$ be a size parameter. Let $\calA$ denote the set of all partial assignments $\alpha \in \{0,1,\star\}^{d}$ for which $|S(\alpha)| = m$ and $w(\alpha) = 1$. Let $\calB$ denote the set of all sorted partial assignments given by refining some element of $\calA$ and sorting. Given $i,j\ge 0$, let $\calB_{i,j}$ denote the set of $\overline{\beta}\in\calB$ for which $|S(\overline{\beta})| = i$ and $w(\overline{\beta}) = j$. For any symmetric function $\psi:\R^d\to\R$, define the function
		\begin{equation}
			\psi^* \triangleq \psi - \sum^m_{i=0} \sum^{m+1-i}_{j=1}  (-1)^{m-i} \cdot \lambda_{i+j} \sum_{\overline{\beta}\in\calB_{i,j}} \psi_{\overline{\beta}}, \qquad \text{for} \ \lambda_k \triangleq \binom{d-k-1}{m-k+1}
		\end{equation}
		Then
		\begin{enumerate}[label=(\alph*)]
			\item $|\calB| \le \binom{d}{m}(d-m)\cdot 3^m$
			\item $\psi^*$ is symmetric
			\item $\psi^*_{\alpha}:\R^d\to\R$ is the identically zero function for all $\alpha \in \calA$.
		\end{enumerate}
	\end{lemma}
	
	\begin{lemma}\label{lem:n3}
		Let \[ \psi(s_1, \dots, s_d; t) = \sum^d_{i=1} \relu\bigg(t - \Big(s_i - \frac{1}{d-1}\sum_{j\neq i} s_j\Big)\bigg) - \relu(dt), \] viewed as a function of $s_1, \dots, s_d$ parameterized by $t$, and let $\psi^*$ be as above. Define $N_3(s_1, \dots, s_d; t) = \psi^*(s_1, \dots, s_d; t)$. Then \begin{enumerate}[label=(\alph*)]
			\item $N_3(s_1, \dots, s_d; t) = 0$ for any $t \in \R$ if between 1 and $m$ of the $s_j$ are 0
			\item $N_3(s_1, \dots, s_d; t) = 0$ for any $s_1, \dots, s_d \in [0,1]^d$ if $t \in \Z \setminus \{0\}$
			\item $N_3$ has size at most $d^{2m}$
			\item $N_3(\underbrace{0, \dots, 0}_{d-1}, s; 0) = s$ for any $s \in [0, \frac{1}{d}]$.
		\end{enumerate}
	\end{lemma}
	
	Before proceeding to the proofs of \cref{lem:main_truncation,,lem:n3}, let us see how to use them to prove \cref{thm:continuous-lift}. 
	
	\begin{proof}[Proof of \cref{thm:continuous-lift}]
		For each $f \in \calC$ given by $f = \sigma \circ h$, let $f^{\liftsym} \in \calC^{\liftsym}$ be given by
		\begin{equation}\label{eq:flift-def}
			f^{\liftsym}(z) = \sum_{t^* \in T} \sigma(t^*) N_3(N_2(z_1), \dots, N_2(z_d);\ h(N_1(z)) - t^*),
		\end{equation} where $N_1$ and $N_2$ are from \cref{lem:n1,,lem:n2}, with the $\delta$ parameter set to $d^{-10m}$, and $N_3$ is from \cref{lem:n3}. This is an $(L+1)$-hidden layer network since $h \circ N_1$ and $N_2$ each have at most $L$ hidden layers, and $N_3$ adds an additional layer. By \cref{lem:n3}(c), the size of this network is $S = d^{\Theta(m)}$. Note that whenever $z$ is such that $N_2(z_1), \dots, N_2(z_d) < 1$, then: \begin{itemize}
			\item $N_1(z) = \thres_q(z)$, and so $h(N_1(z)) = h(\thres_q(z))$ takes only integer values in $T = h(\Z_q^d)$; and
			\item the only $t^*$ for which one of the summands in \cref{eq:flift-def} is potentially nonzero is the one given by $t^* = h(\thres_q(z))$.
		\end{itemize} Thus in this case $f^{\liftsym}$ simplifies to \begin{align} f^{\liftsym}(z) &= \sigma(h(\thres_q(z)))\ N_3(N_2(z_1), \dots, N_2(z_d);\ 0) \\
			&= f(\thres_q(z))\ N_3(N_2(z_1), \dots, N_2(z_d);\ 0). \label{eq:flift-simplified}
		\end{align} Further, for $z$ such that between 1 and $m$ of the $N_2(z_j)$ are 1, we know that $\psi(N_2(z_1), \dots, N_2(z_d);\ t) = 0$ identically (for all $t \in \R$), so in this case $f^{\liftsym}(z) = 0$. And finally, for $z$ such that more than $m$ of the $N_2(z_j)$ are 1, we have no guarantees on the behavior of $f^{\liftsym}$, but as we now show, we have set parameters such that this case occurs only with negligible probability, and we can pretend that $0$ is still a valid label in this case. Indeed, by standard Gaussian anti-concentration, for each coordinate $z_j$ we have $\pr_{z_j}[N_2(z_j) = 1] = \pr_{z_j}[z_j \in \cup_{k} R_k] = \delta = d^{-10m}$. The number of coordinates $j$ for which $N_2(z_j) = 1$ thus follows a binomial distribution $B(d, d^{-10m})$, which has a decreasing pdf with unique mode at $\floor{(d+1)d^{-10m}} = 0$. Thus the probability of having at least $m$ 1s is at most \begin{equation}\label{eq:badcase-prob}
			\sum_{i=m}^{d} \binom{d}{i} (d^{-10m})^{i}(1 - d^{-10m})^{d-i} \leq (d - m + 1) \binom{d}{m} d^{-10m^2} \leq d d^{m} d^{-10m^2} \leq d^{-9m^2}
		\end{equation}
		for sufficiently large $d$. This is negligibly small not only in $d$ but in the size of the network, $S = d^{\Theta(m)}$.
	
		We now describe the reduction. For each labeled example $(x, y)$ that the discrete learner $B$ receives, where $x \sim \unif(\Z_q^d)$ and $y = f(x)$ for an unknown $f \in \calC$, $B$ forms a labeled example $(z, \wt{y})$ for the Gaussian learner $A$ as follows. For each coordinate $j \in [d]$, $z_j$ is drawn from $\calN(0, 1)$ conditioned on $z_j \in I_{x_j}$. Notice that this way $\thres_q(z) = x$, and the marginal distribution on $z$ is exactly $\calN_d$. The modified label is given by \begin{equation}
			\wt{y} = \wt{y}(y, z) = 
			\begin{cases}
				0 & \text{if more than $m$ of the $N_2(z_j)$ are 1} \\
				0 & \text{if between 1 and $m$ of the $N_2(z_j)$ are 1} \\
				y\ N_3(N_2(z_1), \dots, N_2(z_d);\ 0) & \text{otherwise}
			\end{cases} \label{eq:y-tilde}
		\end{equation} Note that in the bottom two cases, $\wt{y} = f^{\liftsym}(z)$ exactly; in the top case $\wt{y}$ is in general inconsistent with $f^{\liftsym}$, but as we have seen, this case occurs with $\negl(S)$ probability. In particular, with overwhelming probability, no $\poly(S)$-time algorithm will ever see non-realizable samples.
	
		So $B$ can feed these new labeled examples $(z, \wt{y})$ to $A$. Suppose $A$ outputs a hypothesis $\hat{f} : \R^d \to \R$ such that $\ex_{z \sim N_d}[(\hat{f}(z) - f^{\liftsym}(z))^2] \leq \epsilon$. We need to show $B$ can convert this hypothesis into a nontrivial one for its discrete problem. We first define a ``good region'' $G \subseteq \R^d$ where $f^{\liftsym}$ is guaranteed to be nonzero and nontrivially related to the original $f$ by saying $z \in G$ iff $N_2(z_1), \dots, N_2(z_{d-1}) = 0$, and $N_2(z_d) \in (\frac{1}{2d}, \frac{1}{d})$. Observe that when $z \in G$, by \cref{eq:flift-simplified} and \cref{lem:n3}(d) we have \begin{align}
			f^{\liftsym}(z) &= f(\thres_q(z)) N_3(N_2(z_1), \dots, N_2(z_{d-1}), N_2(z_d); 0) \\
			&= f(x) N_3(0, \dots, 0, N_2(z_d); 0) \\
			&= y N_2(z_d), \label{eq:flift-goodregion}
		\end{align} where we use the fact that $\thres_q(z) = x$, so that $f(\thres_q(z)) = f(x) = y$. Let us compute the probability mass of $G$. For coordinates $j \in [d-1]$, note that $\pr[N_2(z_j) = 0] = \pr[z_j \notin \cup_k S_k] = 1 - 2\delta = 1 - d^{-\Theta(m)}$. For $z_d$, we need a lower bound on the probability that $N_2(z_d) \in (\frac{1}{2d}, \frac{1}{d})$. Consider the behavior of $N_2$ on just the interval $S_k$ that is closest to the origin (which will be $k = \ceil{q/2}$): it changes linearly from $0$ to $1$ (and again from $1$ to $0$) on $S_k \setminus R_k$. It is not hard to see that $N_2$ takes values in $(\frac{1}{2d}, \frac{1}{d})$ on a $O(1/d)$ fraction of $S_k$. Since the Gaussian pdf will be at least some constant on all of $S_k$, the probability that $z_d$ lands in this fraction of $S_k$ is $\Omega(|S_k|/d) = \Omega(\delta / qd) \geq d^{-\Theta(m)}$. Overall, we get that \[ \pr[z \in R] = \pr\Big[N_2(z_d) \in \Big(\frac{1}{2d}, \frac{1}{d}\Big)\Big] \prod_{j \in [d-1]}\pr[N_2(z_j) = 0]  \geq (1 - d^{-\Theta(m)})^{d-1} d^{-\Theta(m)} = d^{-\Theta(m)}, \] which is still $1/\poly(S)$ and hence non-negligible in the size $S$ of the network.
		
		The discrete learner $B$ can now adapt $\hat{f}$ as follows. Given a fresh test point $x \sim \unif(\Z_q^d)$, the learner forms $z$ such that for each coordinate $j \in [d]$, $z_j$ is drawn from $\calN(0, 1)$ conditioned on $z_j \in I_{x_k}$; for brevity, we shall denote the random variable $z$ conditioned on $x$ (formed in this way) by $z | x$. If $z \in G$, then $B$ predicts $\hat{y} = \frac{\hat{f}(z)}{N_2(z_d)}$ (recall that when $z \in z$, $N_2(z_d) > \frac{1}{2d}$), and otherwise it simply predicts $\wt{y} = \frac{1}{2}$. The square loss of this predictor is given by \begin{align}
			\ex_{x \sim \unif(\Z_q^d)} [(\hat{y} - f(x))^2] &= \ex_{x} \ex_{z|x} [(\hat{y} - f(x))^2] \\
			&= \ex_{x, z|x} [(\hat{y} - f(x))^2 \mid z \in G] \pr[z \in G] + \ex_{x, z|x} [(\hat{y} - f(x))^2 \mid z \notin G] \pr[z \notin G] \\
			&= \ex_{x, z|x} \Big[\Big(\frac{\hat{f}(z)}{N_2(z_d)} - f(x)\Big)^2 \mid z \in G\Big] \pr[z \in G] + \ex_{x, z|x} \Big[\Big(\frac{1}{2} - f(x)\Big)^2 \mid z \notin G\Big] \pr[z \notin G] \\
			&= \ex_{x, z|x} \Big[\Big(\frac{\hat{f}(z)}{N_2(z_d)} - \frac{f^{\liftsym}(z)}{N_2(z_d)}\Big)^2 \mid z \in G\Big] \pr[z \in G] + \ex_{x} \Big[\Big(\frac{1}{2} - f(x)\Big)^2\Big] \pr[z \notin G] \tag{by \cref{eq:flift-goodregion}, when $z \in G$, $f^{\liftsym}(z) = f(x) N_2(z_d)$} \\
			&< 4d^2 \ex_{z}[(\hat{f}(z) - f^{\liftsym}(z))^2 \mid z \in G]\pr[z \in G] + \ex_{x} \Big[\Big(\frac{1}{2} - f(x)\Big)^2\Big] \pr[z \notin G] \tag{when $z \in G$, $N_2(z_d) > \frac{1}{2d}$} \\
			&\leq 4d^2 \ex_{z}[(\hat{f}(z) - f^{\liftsym}(z))^2] + \ex_{x} \Big[\Big(\frac{1}{2} - f(x)\Big)^2\Big] \pr[z \notin G] \\
			&\leq 4d^2 \epsilon + \ex_{x} \Big[\Big(\frac{1}{2} - f(x)\Big)^2\Big] \pr[z \notin G] \\
			&= \ex_{x} \Big[\Big(\frac{1}{2} - f(x)\Big)^2\Big] + 4d^2 \epsilon - \ex_{x} \Big[\Big(\frac{1}{2} - f(x)\Big)^2\Big] \pr[z \in G].
		\end{align} In the case of the hard classes $\calC$ that we consider, we may assume without loss of generality that $\ex_{x \sim \unif(\Z_q^d)}[(\frac{1}{2} - f(x))^2] \geq 1/\poly(d)$, since otherwise the problem of learning $\calC$ is trivial (in fact, in our applications we will have $\ex_{x \sim \unif(\Z_q^d)}[(\frac{1}{2} - f(x))^2] = \Theta(1)$). This means that by taking \[ \epsilon = \pr[z \in G]/\poly(d) = d^{-\Theta(m)}/\poly(d) = d^{-\Theta(m)} \] sufficiently small (but still $1/\poly(S)$), we may ensure that the square loss of the discrete learner $B$ is at most $\ex_{x \sim \unif(\Z_q^d)}[(\frac{1}{2} - f(x))^2] - d^{-\Theta(m)}$, as desired.
	\end{proof}
	
	\begin{remark}\label{rem:distribution}		
		The only property of the Gaussian $\calN(0, \Id_d)$ used crucially in the proof above is that it is a product distribution $P = \otimes_{i \in [d]} P_i$ where each $P_i$ is suitably anti-concentrated. By some simple changes to the parameters of $N_1$, $N_2$ and $N_3$ (depending on $P$), the proof can be made to work more generally for such distributions $P$.
	\end{remark}
	
	\subsection{Proofs of \cref{lem:main_truncation,,lem:n3}}
	\label{subsec:truncation}
	
	We now detail the proofs involved in the construction of the gadget $N_3$.
	
	\begin{proof}[Proof of \cref{lem:main_truncation}]
		Note that $|\calA| = \binom{d}{m}(d-m)$. Any partial assignment $\beta$ has at most $3^{|S(\beta)|}$ refinements, and $\calB$ is a subset of all refinements of partial assignments from $\calA$, so $|\calB| \le \binom{d}{m}(d-m)\cdot 3^m$.
		
		For the remaining parts of the lemma, it will be useful to observe that $\calB$ consists exactly of all partial assignments with $i$ free variables and $j$ 1s for any $0 \le i \le m$ and $j \ge 1$ satisfying $i + j \le m + 1$.
		
		To prove the second part of the lemma, it suffices to show that
		\begin{equation}
			\sum_{\overline{\beta}\in\calB_{i,j}} h_{\overline{\beta}} \label{eq:singleij}
		\end{equation}
		is symmetric for all $i,j$. 
		As transpositions generate the symmetric group on $d$ elements, it suffices to show that \eqref{eq:singleij} is invariant under swapping two input coordinates, call them $a,b\in[d]$. For all $\overline{\beta}\in\calB_{i,j}$ for which $a,b$ are either both present or both absent in $S(\overline{\beta})$, this clearly does not affect the value of $h_{\overline{\beta}}$. Now consider the set $S_a$ (resp. $S_b$) of partial assignments $\overline{\beta}\in\calB_{i,j}$ for which only $a$ (resp. only $b$) is present in $S(\overline{\beta})$. There is a clear bijection $f: S_a \to S_b$: given $\overline{\beta}\in S_a$, swap the $a$- and $b$-th entries, and vice-versa, and for any $\overline{\beta}\in S_a$, the function $h_{\overline{\beta}} + h_{f(\overline{\beta})}$ is unaffected by the swapping of input coordinates $a,b$. This concludes the proof of the second part of the lemma.
		
		Finally, to prove the third part of the lemma, it suffices to verify it for a single $\alpha\in\calA$, as $h^*$ is symmetric. So consider $\overline{\alpha} = \brc{1,0,\cdots,0,\star,\cdots,\star}$. We apply \eqref{eq:chain} to get
		\begin{align}
			h^*_{\overline{\alpha}} &= h_{\overline{\alpha}} - \sum^m_{i=0} \sum^{m+1-i}_{j=1}  (-1)^{m-i} \cdot \lambda_{i+j}  \sum_{\overline{\beta}\in\calB_{i,j}} h_{\overline{\alpha}\searrow\overline{\beta}} \\
			&= h_{\overline{\alpha}} - \sum_{\overline{\gamma} \in\calB\cap \App(\alpha) \ \text{sorted}} h_{\overline{\gamma}} \cdot \sum^m_{i=0}\sum^{m+1-j}_{j=1} (-1)^{m-i} \cdot \lambda_{i+j} \sum_{\overline{\beta}\in\calB_{i,j}} \ind[\overline{\overline{\alpha}\searrow\overline{\beta}} = \overline{\gamma}] \label{eq:main}
		\end{align}
		
		Note that for $\overline{\gamma} = \overline{\alpha}$, the only $\overline{\beta}\in\calB$ for which $\overline{\overline{\alpha}\searrow\overline{\beta}} = \overline{\gamma}$ is $\overline{\beta} = \overline{\alpha}$. Indeed, for $\overline{\beta}$ to be such that $\overline{\overline{\alpha}\searrow\overline{\beta}} = \overline{\alpha}$, it must have $S(\overline{\beta}) = S(\overline{\alpha})$ and exactly one 1, from which it follows that $\overline{\beta} = \overline{\alpha}$. Since $\overline{\alpha} \in \calB_{m, 1}$, its coefficient in  \eqref{eq:main} is given by
		\begin{equation}
			(-1)^{m-m}\cdot \lambda_{m+1} = 1,
		\end{equation}
		and so the $h_\alpha$ in \eqref{eq:main} cancels with the $\overline{\gamma} = \overline{\alpha}$-th summand in \eqref{eq:main}.
		
		In the rest of the proof, we can thus focus on sorted $\overline{\gamma}\in \calB\cap \App(\alpha) \backslash \brc{\overline{\alpha}}$. Note that such $\overline{\gamma}$ satisfy
		\begin{equation}
			|S(\overline{\gamma})| < m \label{eq:assume}.
		\end{equation}
		To see this, recall that any $\overline{\gamma} \in \calB$ with $|S(\overline{\gamma})| = m$ must have exactly one 1, and since $\overline{\gamma} \in \App(\overline{\alpha})$ it must be that $\overline{\gamma}$ must have $S(\overline{\gamma}) = S(\overline{\alpha})$ and so $\overline{\gamma} = \overline{\alpha}$.
		
		Observe that we must have $\overline{\gamma}_1 = 1$. Indeed, it cannot be 0 because $\overline{\gamma}$ is sorted and has at least one 1. It also cannot be $\star$. To see this, consider any $\overline{\beta}$ for which $\overline{\overline{\alpha}\searrow\overline{\beta}} = \overline{\gamma}$. If we had $\overline{\beta}_1 \neq \star$, then clearly $\overline{\gamma}_1 \neq \star$. If we had $\overline{\beta}_1 = \star$, then $(\overline{\alpha}\searrow\overline{\beta})_1 = 1$ (as $\overline{\alpha}_1 = 1$), so $\overline{\gamma} = \overline{\overline{\alpha}\searrow\overline{\beta}}$ must also have first entry given by 1.
		
		
		We are now ready to calculate the coefficient of $h_{\overline{\gamma}}$ (for each $\overline{\gamma}\in \calB\cap \App(\alpha) \backslash \brc{\overline{\alpha}}$) in \eqref{eq:main} by adding the coefficients of all the $\overline{\beta}\in\calB$ for which 
		\begin{equation}
			\overline{\overline{\alpha}\searrow\overline{\beta}} = \overline{\gamma}. \label{eq:inverse}
		\end{equation}
		
		First let us consider the contribution of $\overline{\beta}\in\calB$ for which $\overline{\beta}_1 = 1$. Observe that such $\overline{\beta}$ must have exactly $w(\overline{\gamma})$ 1s. Furthermore, such a $\overline{\beta}$ is an element of $\calB$ if and only if it has at most $m + 1 - w(\overline{\gamma})$ free variables, and the set of free variables in $\overline{\beta}$ must be $S(\overline{\gamma})\cup V$ where $V$ is any subset of $[d]\backslash(\brc{1} \cup S(\overline{\alpha}))$. The contribution of all such $\overline{\beta}$ to the coefficient of $h_{\overline{\gamma}}$ in \eqref{eq:main} is thus
		\begin{equation}
			\sum^{m+1-w(\overline{\gamma})}_{i = |S(\overline{\gamma})|} (-1)^{m-i}\cdot \lambda_{i+w(\overline{\gamma})} \cdot \binom{d - m - 1}{i - |S(\overline{\gamma})|}, \label{eq:contrib1}
		\end{equation}
		where here the index $i$ denotes the total number of free variables in $\overline{\beta}$, and the factor of $\binom{d-m-1}{i - |S(\overline{\gamma})|}$ is the number of ways to choose $V$.
		
		It remains to consider the contribution from $\overline{\beta}\in\calB$ for which $\overline{\beta}_1 \neq 1$. First note that clearly we cannot have $\overline{\beta}_1 = 0$, as $\overline{\beta}$ is sorted and has at least one 1 because it lies in $\calB$. The only possibility is $\overline{\beta}_1 = \star$, which we split into two cases based on $w(\overline{\gamma})$.
		
		\begin{description}[leftmargin=0cm]
			\item[Case 1: $w(\overline{\gamma}) = 1$.] In this case, we claim that there are no $\overline{\beta}\in\calB$ simultaneously satisfying \eqref{eq:inverse} and $\overline{\beta}_1 = \star$. Suppose to the contrary. Then such a $\overline{\beta}_1$ must have at least one 1 in some other entry (as $\overline{\beta}\in\calB$), but this would imply that the resolution $\overline{\alpha}\searrow\overline{\beta}$ has at least two 1s, a contradiction. The total coefficient of $h_{\overline{\gamma}}$ in this case is thus exactly given by \eqref{eq:contrib1}. Upon substituting $w(\overline{\gamma}) = 1$, this simplifies to
			\begin{equation}
				\sum^{m+1-w(\overline{\gamma})}_{i = |S(\overline{\gamma})|} (-1)^{m-i}\cdot \lambda_{i+1} \cdot \binom{d - m - 1}{i - |S(\overline{\gamma})|} = \sum^{m+1-w(\overline{\gamma})}_{i = |S(\overline{\gamma})|} (-1)^{m-i}\cdot \binom{d-i-2}{d-m-2} \cdot \binom{d - m - 1}{i - |S(\overline{\gamma})|} = 0,
			\end{equation}
			where in the last step we use \cref{lem:altsum} (which we can apply because of \eqref{eq:assume}).
			
			\item[Case 2: $w(\overline{\gamma}) > 1$.] Observe that we must have $w(\overline{\beta}) = w(\overline{\gamma}) - 1$ (as the only entry of $\overline{\alpha}$ equal to 1 is the first entry, and the first entry of $\overline{\beta}$ is $\star$). As $w(\overline{\gamma}) - 1 > 0$ in the current case, such a $\overline{\beta}$ is an element of $\calB$ if and only if it has at most $m + 2 - w(\overline{\gamma})$ free variables, and the set of free variables in $\overline{\beta}$ must be $\brc{1}\cup S(\overline{\gamma}) \cup V$ where $V$ is any subset of $[d]\backslash(\brc{1}\cup S(\overline{\alpha}))$. Thus in this second case, the contribution of all $\overline{\beta}$ with $\overline{\beta}_1 = \star$ to the coefficient of $h_{\overline{\gamma}}$ in \eqref{eq:main} is
			\begin{equation}
				\sum^{m+2-w(\overline{\gamma})}_{i = |S(\overline{\gamma})|+1} (-1)^{m-i}\cdot \lambda_{i+w(\overline{\gamma})-1}\cdot \binom{d-m-1}{i-|S(\overline{\gamma})|-1} = \sum^{m+1-w(\overline{\gamma})}_{j = |S(\overline{\gamma})|} (-1)^{m-j-1}\cdot \lambda_{j+w(\overline{\gamma})}\cdot \binom{d-m-1}{j-|S(\overline{\gamma})|}, \label{eq:contrib2}
			\end{equation}
			where here the index $i$ denotes the total number of free variables in $\overline{\beta}$, the factor of $\binom{d-m-1}{i - |S(\overline{\gamma})| - 1}$ is the number of ways to choose $V$ (note that $|V| = i - |S(\overline{\gamma})| - 1$), and in the second expression we made the change of variable $j = i - 1$. We conclude that in this case, the coefficient of $h_{\overline{\gamma}}$ in \eqref{eq:main} is given by the sum of \eqref{eq:contrib1} and \eqref{eq:contrib2}, which is 0.
		\end{description}
		
		Overall, we conclude that the entire RHS of \eqref{eq:main} vanishes for $\alpha \in \calA$, proving the third part of the lemma.
	\end{proof}
	
	The next lemma formally constructs $N_3$ and verifies that it has the required properties, is of acceptable size, and that it takes on nonzero values on a significant part of its domain.
	
	\begin{proof}[Proof of \cref{lem:n3}]
		Part (a) follows directly from \cref{lem:main_truncation}(c). Part (b) follows by verifying that for any $t \in \mathbb{Z} \setminus \{0\}$, $\psi(s_1, \dots, s_d; t) = 0$ for any $s_1, \dots, s_d \in [0, 1]^d$; this means that $\psi^*$, which is a combination of partial restrictions of $\psi$, also vanishes for such $t$. First suppose that $t$ is a positive integer. Observe that $t \ge 1$ while $s_i - \frac{1}{d-1}\sum_{j\neq i} s_j \in [-1,1]$, so each ReLU in the definition of $\psi$ is activated and we get
		\begin{equation}
			\psi(s_1,\ldots,s_d;t) = \sum^d_{i=1} \left[t - \left(s_i - \frac{1}{d-1}\sum_{j\neq i}s_j\right)\right] - dt = -\sum^d_{i=1}\left(s_i - \frac{1}{d-1}\sum_{j\neq i}s_j\right) = 0.
		\end{equation}
		Next suppose that $t$ is a negative integer. Then $t \le -1$ while $s_i - \frac{1}{d-1}\sum_{j\neq i} s_j \in [-1,1]$, so each ReLU in the definition of $h$ is inactive and we get $\psi(s_1,\ldots,s_d;t) = 0$.
		
		For part (c), observe that by the size bound in \cref{lem:main_truncation}(a) and the fact that $\psi$ contains $O(d)$ ReLUs, the size of $N_3$ may be bounded by \[ S \le O(d)\cdot(\binom{d}{m}(d-m)\cdot 3^m + 1) \le O(d)(\frac{d^{m+1}\cdot 3^m}{m!} + 1) \le d^{m+2} \le d^{2m} \] for $m$ larger than some absolute constant.
		
		It remains to prove part (d). For brevity, we will omit the parameter $t$ and just refer to $\psi(0,\ldots,0,s;t)$ and $\psi^*(0,\ldots,0,s;t)$ as $\psi(0,\ldots,0,s)$ and $\psi^*(0,\ldots,0,s)$. We first compute $\psi(0,\ldots,0,s)$: for $s\in[0,1]$,
		\begin{equation}
			\psi(0,\ldots,0,s) = \relu(-s) + (d - 1)\relu(\frac{1}{d-1}\cdot s) = s.
		\end{equation}
		Next, for any $\overline{\beta}\in\calB$, if $w(\overline{\beta}) = j$ for some $0\le j\le m+1$, then if $\overline{\beta}_d = \star$,
		\begin{align}
			\MoveEqLeft \psi_{\overline{\beta}}(0,\ldots,0,s) \\
			&= \psi(\underbrace{1,\ldots,1}_j,\underbrace{0,\cdots 0}_{d-j-1},s) \\
			&= j\cdot\relu\left(-1 + \frac{1}{d-1}(j-1+s)\right) + (d - j-1) \cdot \relu\left(\frac{1}{d-1}\cdot j + \frac{1}{d-1}\cdot s\right) \\ &\quad + \relu\left(-s + \frac{1}{d-1}\cdot j\right) \\
			&= \frac{d-j-1}{d-1}\cdot (j + s) + \relu\left(-s + \frac{1}{d-1}\cdot j\right)
			\intertext{Note that when $s\in[0,1/(d-1)]$, because $j \ge 1$ (as $\overline{\beta}\in\calB$) this simplifies to }
			&= \frac{(d-j-s)j}{d-1}.
		\end{align}
		On the other hand, if $\overline{\beta}_d \in\brc{0,1}$, then
		\begin{align}
			\psi_{\overline{\beta}}(0,\ldots,0,s) &= \psi(\underbrace{1,\ldots,1}_j,\underbrace{0,\cdots 0}_{d-j}) \\
			&= j\cdot \relu\left(-1 + \frac{1}{d-1}(j-1)\right) + (d-j)\cdot \relu\left(\frac{1}{d-1}\cdot j\right)  = \frac{(d-j)j}{d-1}.
		\end{align}
		As there are $\binom{d-1}{i-1}$ (resp. $\binom{d-1}{i}$) partial assignments in $\calB_{i,j}$ for which $\overline{\beta}_d = \star$ (resp. $\overline{\beta}_d \in\brc{0,1}$), we can thus explicitly compute $h^*(0,\ldots,0,s)$ for $s\in[0,1/(d-1)]$ to be
		\begin{equation}
			\psi(0,\ldots,0,s) -\sum^m_{i=0}\sum^{m+1-i}_{j=1}(-1)^{m-i}\binom{d-i-j-1}{m-i-j+1}\left(\binom{d-1}{i-1}\cdot \frac{(d-j-s)j}{d-1} + \binom{d-1}{i}\cdot \frac{(d-j)j}{d-1}\right).
		\end{equation}
		By Lemma~\ref{lem:doublesum}, the double sum is equal to sero, so $h^*(0,\ldots,0,s) = h(0,\ldots,0,s) = s$ for $s\in[0,1/(d-1)]$ as claimed.
	\end{proof}
	
	\section{Statistical Query Lower Bound}
	
	We prove a superpolynomial SQ lower bound (for general queries as opposed to only correlational or Lipschitz queries) for weakly learning two-hidden-layer ReLU networks under the standard Gaussian.
	
	\begin{theorem}\label{thm:depth2-sq-lb}
		Fix any $\alpha \in (0, 1)$. Any SQ learner capable of learning $\poly(d)$-sized two-hidden-layer ReLU networks under $\calN(0, \Id_d)$ up to squared loss $\epsilon$ (for some sufficiently small $\epsilon = 1/\poly(d)$) using bounded queries of tolerance $\tau \geq 2^{-(\log d)^{2 - \alpha}}$ must use at least $\Omega(2^{2^{(\log d)^{\alpha}}} \tau^2) = d^{\omega(1)} \tau^2$ such queries.
	\end{theorem}
	
	
	For instance, taking $\alpha=\frac{1}{2}$ gives a slightly subexponential (but super-quasipolynomial) in $d$ query lower bound for queries of tolerance at least inverse quasipolynomial in $d$.
	
	This theorem is proven using the following key reduction, which adapts the compressed DV lift (\cref{thm:continuous-lift}) to the SQ setting.
	
	\begin{theorem}\label{thm:sq-continuous-lift}
		Let $q = \poly(d)$ be a modulus, and let $m = m(d) = \omega_d(1)$ be a size parameter. Let $\calC$ be a class of compressible $L$-hidden-layer $\poly(d)$-sized ReLU networks mapping $\Z_q^d$ to $[0, 1]$, and let $\calC^{\liftsym}$ be the lifted class of $(L+1)$-hidden-layer $d^{\Theta(m)}$-sized ReLU networks corresponding to $\calC$, mapping $\R^d$ to $\R$ (as in \cref{thm:continuous-lift}). Suppose there is an SQ learner $A$ capable of learning $\calC^{\liftsym}$ over $\calN(0, \Id_d)$ up to squared loss $d^{-\Theta(m)}$ using queries of tolerance $\tau$, where $\tau \geq d^{-\Theta(m^2)}$. Then there is an SQ learner $B$ that, using the same number of queries of tolerance $\tau/2$, produces a weak predictor $\wt{B}$ for $\calC$ over $\unif(\Z_q^d)$ with advantage $d^{-\Theta(m)}$ over guessing the constant $1/2$ (in expectation over both the data and the internal randomness of $\wt{B}$).
	\end{theorem}
	\begin{proof}
		Recall that $B$ is given SQ access to a distribution of pairs $(x, y)$ where $x \sim \unif(Z_q^d)$ and $y = f(x)$ for an unknown $f \in \calC$. $A$ can request estimates $\ex[\phi(x, y)] \pm \tau$ for arbitrary bounded queries $\phi : \Z_q^d \times [0,1] \to [-1, 1]$ and any desired $\tau$. We know that given $(x, y)$, the distribution of $(z, \wt{y})$, where $z = z(x)$ is defined by drawing each $z_j$ from $\calN(0, 1)$ conditioned on $z_j \in I_{x_j}$ and $\wt{y} = \wt{y}(y, z)$ is as in \cref{eq:y-tilde}), is consistent with some $f^{\liftsym} \in \calC^{\liftsym}$ except on a region of probability mass at most $d^{-9m^2}$ (recall \cref{eq:badcase-prob}). Suppose we could simulate SQ access to the distribution of $(z, f^{\liftsym}(z))$ using only SQ access to that of $(x, f(x))$. Then by the argument in \cref{thm:continuous-lift}, simulating $A$ on the $(z, f^{\liftsym}(z))$ distribution would give us a weak predictor $\wt{B}$ for the distribution of $(x, f(x))$, satisfying \[ \ex \big[\big(\wt{B}(x) - f(x)\big)^2\big] < \ex \big[ \big(\frac{1}{2} - f(x)\big)^2 \big] - d^{-\Theta(m)}. \]
		
		What we must describe is how $B$ can simulate $A$'s statistical queries. Say $A$ requests an estimate $\ex_z[\phi(z, f^{\liftsym}(z))] \pm \tau$ for some query $\phi : \R^d \times \R \to [-1, 1]$. Consider the query $\wt{\phi} : \Z_q^d \times [0,1] \to [-1, 1]$ given by $\wt{\phi}(x, y) = \ex_{z(x)}[\phi(z(x), \wt{y}(y, z(x)))]$. This function can be computed without any additional SQs, since the distribution of $(z, \wt{y}) = (z(x), \wt{y}(y, z(x)))$, given $(x, y)$, is fully determined and known to $B$. Observe that \begin{equation} \label{eq:sq-query1}
			\ex_{x, y} \wt{\phi}(x, y) = \ex_{x, z(x)}[\phi(z(x), \wt{y}(y, z(x)))] = \ex_{z,\wt{y}}[\phi(z, \wt{y})]. \end{equation}
		We must also account for the difference between $\ex_z[\phi(z, f^{\liftsym}(z))]$ and $\ex_{z, \wt{y}}[\phi(z, \wt{y})]$. But because the distributions only differ on a region of mass $d^{-9m^2}$ and $\phi$ is bounded, we have \begin{equation}\label{eq:sq-query2}
			\Big| \ex_z[\phi(z, f^{\liftsym}(z))] - \ex_{z, \wt{y}}[\phi(z, \wt{y})] \Big| \leq \Theta(d^{-9m^2}) \leq \frac{\tau}{2}
		\end{equation} since we assumed $\tau \geq d^{-\Theta(m^2)}$. Putting together \eqref{eq:sq-query1} and \eqref{eq:sq-query2}, we see that $B$ can simulate $A$'s query $\phi$ to within tolerance $\tau$ by querying $\wt{\phi}$ with tolerance $\tau/2$.
	\end{proof}
	
	Again, by a padding argument we can obtain a corollary similar to \cref{cor:padded-lift}, for which we omit the formal statement. We will use such an argument in the proof of \cref{thm:depth2-sq-lb}.

	\subsection{SQ lower bound via parities}
	
	We can obtain an SQ lower bound for two-hidden-layer ReLU networks by lifting the problem of learning parities under $U_d$, which is well-known to require exponentially many queries. More precisely, we show that an SQ learner for two-hidden-layer ReLU networks would yield an SQ algorithm for the problem of distinguishing an unknown parity from random labels.
	
	\begin{theorem}[\cite{kearns1998efficient,blum1994weakly}]\label{thm:parity-lb}
		Consider an SQ algorithm given SQ access either to the distribution of labeled pairs $(x, y)$ where $x \sim U_d$ and $y = \chi_S(x)$ for an unknown parity $\chi_S$ or to the randomly labeled distribution $U_d \times \unif\{0, 1\}$. Any algorithm capable of distinguishing between the two cases with probability $2/3$ using queries of tolerance $\tau$ requires at least $\Omega(2^d \tau^2)$ such queries.
	\end{theorem}
	
	\begin{lemma}\label{lem:parity-impl}
		For every $S \subseteq [d]$, the parity function $\chi_S: \{0,1\}^d \to \brc{0,1}$ can be implemented as a compressible one-hidden-layer ReLU network of $\poly(d)$ size.
	\end{lemma}
	\begin{proof}
		Recall that $\chi_S(x)$ evaluates to 1 if the Hamming weight of the bits of $x$ in $S$ is odd, and 0 otherwise, so that $\chi_S(x) = \sigma(\sum_{j \in S} x_j)$. This satisfies the definition of a compressible one-hidden-layer network with the inner depth-0 network being $x \mapsto \sum_{j \in S} x_j$ and $\sigma(t) = \ind[t \text{ is odd}]$.
	\end{proof}
	
	We can now supply one proof of \cref{thm:depth2-sq-lb}.
	\begin{proof}[First proof of \cref{thm:depth2-sq-lb}]
		Let $m = m(d) = \log^c d$ for $c = \frac{1}{\alpha} - 1$, and let $d' = d^m = 2^{\log^{c+1} d}$, so that $d = 2^{\log^{1/(1+c)} d'}$. By \cref{lem:parity-impl}, the class $\calC$ of parities on $\{0,1\}^d$ can be implemented by compressible one-hidden-layer $\poly(d)$-sized ReLU networks, and so the lifted class $\calC^{\liftsym}$ can be implemented by two-hidden-layer $d^{\Theta(m)}$-sized ReLU networks over $\R^d$. A padding argument lets us embed these classes into dimension $d'$. By using the predictor from \cref{thm:sq-continuous-lift} (with $q = 2$), we could obtain an SQ algorithm capable of distinguishing parities from random labels using queries of tolerance $\tau/2$, assuming $\tau \geq d^{-\Theta(m^2)} = 2^{-\log^{2c+1}d} = 2^{-\log^{\frac{2c+1}{c+1}}d'}$. By \cref{thm:parity-lb}, the lower bound for learning parities is $\Omega(2^d \tau^2) = \Omega(2^{2^{\log^{1/(1+c)} d'}} \tau^2)$. Substituting $\alpha = \frac{1}{1+c}$ gives the result.
	\end{proof}
	
	But the SQ lower bound obtained this way via parities is somewhat unconvincing since there is a non-SQ algorithm capable of learning the lifted function class obtained from parities. Indeed, suppose we are given examples $(z, f^{\liftsym}(z))$ where $f$ is an unknown parity. We know that whenever $z$ lands in the ``good region'' $G$ from the proof of \cref{thm:continuous-lift} (which happens with non-negligible probability), we have $f^{\liftsym}(z) = f(\sgn(z))N_2(z)$ (recall \cref{eq:flift-goodregion}). This means we can simply filter out all $z \notin G$ and form a clean data set of labeled points $(\sgn(z), f(\sgn(z)))$. The unknown $f$ (and hence $f^{\liftsym}$) can now be learnt by simple Gaussian elimination. In order to give a more convincing lower bound, we now provide an alternative proof based on $\lwr$.
	
	\subsection{SQ lower bound via the LWR functions}
	
	Here we provide an alternative proof of \cref{thm:depth2-sq-lb} using the $\lwr$ functions. The hard function class obtained this way is not only \emph{unconditionally} hard for SQ algorithms, it is arguably hard for non-SQ algorithms as well, since $\lwr$ is believed to be cryptographically hard.
	
	We begin by stating an SQ lower bound for the $\lwr$ functions. This theorem is proven in \cref{app:lwr-sq} using a general formulation in terms of pairwise independent function families that may be of independent interest, communicated to us by Bogdanov \cite{bogdanov2021personal}. 
	
	\begin{theorem}\label{thm:lwr-sq-lb}
		Let $\calC_{\lwr}$ denote the $\lwr_{n, p, q}$ function class. Any SQ learner capable of learning $\calC_{\lwr}$ up to squared loss $1/16$ under $\unif(\Z_q^n)$ using queries of tolerance $\tau$ requires at least  $\Omega(q^{n-1} \tau^2)$ such queries.
	\end{theorem}
	
	The following lemma shows that the $\lwr$ functions may be realized as compressible one-hidden-layer ReLU networks.
	
	\begin{lemma}\label{lem:lwr-impl}
		For every $w \in \Z_q^n$, the $\lwr$ function $f_w : \Z_q^n \to \Z_p / p$ can be implemented as a compressible one-hidden-layer ReLU network of size $O(q^2 n)$.
	\end{lemma}
	\begin{proof}
		By definition, we have $f_w(x) = \frac{1}{p} \round{(w \cdot x) \bmod q}_p$, which is a compressible one-hidden-layer ReLU network with the inner depth-0 network (i.e., affine function) being $w \mapsto w \cdot x$ and $\sigma(t) = \frac{1}{p}\round{t \bmod q}_p$. The size bound follows by observing that for any $x \in \Z_q^n$, the quantity $w \cdot x$ is an integer in $\{0, \dots, q^2 n\}$.
	\end{proof}
	
	We are ready for an alternative proof of \cref{thm:depth2-sq-lb}.
	
	\begin{proof}[Alternative proof of \cref{thm:depth2-sq-lb}]
		Let $n$ be the security parameter, and fix moduli $p, q \geq 1$ such that $p, q = \poly(n)$ and $p/q = \poly(n)$. Let $d = n$, so that the SQ lower bound from \cref{thm:lwr-sq-lb} is $\Omega(q^{n-1}) = d^{\Omega(d)} = 2^{\wt{\Omega}(d)}$. Let $m = m(d) = \log^c d$ for $c = \frac{1}{\alpha} - 1$, and let $d' = d^m = 2^{\log^{c+1} d}$, so that $d = 2^{\log^{1/(1+c)} d'}$. By \cref{lem:lwr-impl}, the $\lwr_{n,p,q}$ function class $\calC_{\lwr}$ is implementable by one-hidden-layer ReLU networks over $\Z_q^d$ of size $\poly(n) = \poly(d)$. The result now follows by \cref{thm:sq-continuous-lift} and the same padding argument as in the proof based on parities.
	\end{proof}

	\section{Cryptographic Hardness Based on LWR}
	
	In this section we show hardness of learning two-hidden-layer  ReLU networks over Gaussian inputs based on $\lwr$. This is a direct application of the compressed DV lift (\cref{thm:continuous-lift}) to the $\lwr$ problem, which is by definition a hard learning problem over $\unif(\Z_q^n)$, or equivalently $\unif(\Z_q^d)$ with $d = n$.
	
	\begin{theorem}\label{thm:depth2-crypto-hardness}
		Let $n$ be the security parameter, and fix moduli $p, q \geq 1$ such that $p, q = \poly(n)$ and $p/q = \poly(n)$. Let $d = n$. Let $c > 0$, $m = m(d) = \log^c d$ and $d' = d^m$. Suppose there exists a $\poly(d')$-time algorithm capable of learning $\poly(d')$-sized depth-2 ReLU networks under $\calN(0, \Id_{d'})$ up to squared loss $1/\poly(d')$. Then there exists a $\poly(d') = 2^{\Theta(\log^{1+c} n)}$ time algorithm for $\lwr_{n, p, q}$.
	\end{theorem}
	
	\begin{proof}[Proof of \cref{thm:depth2-crypto-hardness}]
		By \cref{lem:lwr-impl}, we know that the class $\calC_\lwr$ is implementable by compressible $\poly(d)$-sized one-hidden-layer ReLU networks over $\Z_q^d$, or, after padding, over $\Z_q^{d'}$. Let $\calC^{\liftsym}_\lwr$ denote the corresponding lifted class of $\poly(d')$-sized two-hidden-layer ReLU networks, padded to have domain $\R^{d'}$. Applying \cref{cor:padded-lift} to the assumed learner for $\calC^{\liftsym}_\lwr$, we obtain a $\poly(d')$-time weak predictor predictor for $\calC_\lwr$, which readily yields a corresponding distinguisher for the $\lwr_{n, p, q}$ problem. Using the facts that $d' = d^m = 2^{\log^{1+c} d}$ and $d = n$, we may translate $\poly(d')$ into $2^{\Theta(\log^{1+c} n)}$, yielding the result.
	\end{proof}
	
	\begin{remark}\label{rem:crypto-quasipoly}
		Note that the choice of $m = m(d) = \log^c d$ in \cref{thm:depth2-crypto-hardness} is purely for simplicity. By picking $m(d) = \omega_d(1)$ to be a suitably slow-glowing function of $d$, such as $\log^* d$, we can obtain a running time for the final $\lwr$ algorithm that is as mildly superpolynomial as we like.
	\end{remark}
	
	In addition, as an immediate corollary of \cref{lem:lwr-impl}, we also obtain a hardness result for one-hidden-layer networks under $\unif\{0,1\}^d$, improving on the hardness result of \cite{daniely2021local} (see Theorem 3.4 therein) for \emph{two}-hidden-layer networks under $\unif\{0,1\}^d$. For this application, we let $d = n \log q = \wt{O}(n)$, so that we may identify the domain $\Z_q^n$ with $\{0,1\}^{d}$ via the binary representation. This also identifies $\unif(\Z_q^n)$ with $\unif\{0,1\}^{d}$.
	\begin{corollary}\label{thm:uniform}
		Let $n, p, q$ be such that $p, q = \poly(n)$ and $p/q = \poly(n)$, and let $d = n \log q = \wt{O}(n)$. Suppose there exists an efficient algorithm for learning $\poly(d)$-sized one-hidden-layer ReLU networks under $U_d$ up to squared loss $1/4$. Then there exists an efficient algorithm for $\lwr_{n, p, q}$.
	\end{corollary}

	\section{Hardness of Learning using Label Queries}
	\label{sec:query}
	
	The main result of this section is to show hardness of learning constant-depth ReLU networks over Gaussians from label queries:
	
	\begin{theorem}\label{thm:model_hardness}
		Assume there exists a family of PRFs mapping $\{0,1\}^d$ to $\{0, 1\}$ implemented by $\poly(d)$-sized $L$-hidden-layer ReLU networks. Then there does not exist an efficient learner that, given query access to an unknown $\poly(d)$-sized $(L+2)$-hidden-layer ReLU network $f : \R^d \to \R$, is able to output a hypothesis $h : \R^d \to \R$ such that $\ex_{z \sim \calN(0, \Id_d)} [(h(z) - f(z))^2] \leq 1/16$.
	\end{theorem}
	
	We first recall the classical connection between pseudorandom functions and learning from label queries (also known as membership queries in the Boolean setting), due to Valiant~\cite{valiant1984theory} (see e.g.\ \cite[Proposition 12]{bogdanov2017pseudorandom} for a modern exposition).
	
	\begin{lemma}\label{lem:prf-learning-hardness}
		Let $\calC = \{f_s\}$ be a family of PRFs from $\{0,1\}^d$ to $\{0,1\}$ indexed by the key $s$. Then there cannot exist an efficient learner $L$ that, given query access to an unknown $f_s \in \calC$, satisfies \[ \pr_{x, s}[L(x) = f_s(x)] \geq \frac{1}{2} + \frac{1}{\poly(d)}, \] where the probability is taken over the random key $s$, the internal randomness of $A$, and a random test point $x \sim \unif\{0,1\}^d$.
	\end{lemma}
	
	There exist multiple candidate constructions of PRF families in the class $\mathsf{TC}^0$ of constant-depth Boolean circuits built with AND, OR, NOT and threshold (or equivalently majority) gates. Because the majority gate can be simulated by a linear combination of ReLUs similar to $N_1$ from \cref{lem:n1}, any $\mathsf{TC}^0_L$ (meaning depth-$L$) function $f : \{0,1\}^d \to \{0,1\}$ may be implemented as a $\poly(d)$-sized $L$-hidden-layer ReLU network (see e.g.\ \cite[Lemma A.3]{vardi2021size}\footnote{Note that what the authors term a depth-$(L+1)$ network is in fact an $L$-hidden-layer network in our terminology.}). Thus we may leverage the following candidate PRF constructions in $\mathsf{TC}^0$ for our hardness result: \begin{itemize}
		\item PRFs in $\mathsf{TC}^0_4$ based on the decisional Diffie-Hellman (DDH) assumption \cite{krause2001pseudorandom} (improving on \cite{naor1997number}), yielding hardness for depth-6 ReLU networks
		\item PRFs in $\mathsf{TC}^0$ based on Learning with Errors \cite{banerjee2012pseudorandom,banerjee2014new}, yielding hardness for depth-$O(1)$ ReLU networks
	\end{itemize} Note that depth 4 is the shallowest depth for which we have candidate PRF constructions based on widely-believed assumptions, and the question of whether there exist PRFs in $\mathsf{TC}^0_3$ is a longstanding open question in circuit complexity \cite{razborov1992small,hajnal1993threshold,razborov1997natural,krause2001pseudorandom}. Under less widely-believed assumptions, \cite{boneh2018exploring} have also proposed candidate PRFs in $\mathsf{ACC}^0_3$.
	
	We can now complete the proof of \cref{thm:model_hardness}. Since pseudorandom functions are not necessarily compressibile, we will simply use the original DV lift (\cref{thm:og-dv}).
	
	\begin{proof}[Proof of Theorem~\ref{thm:model_hardness}]
		Let $f_s : \{0,1\}^d \to \{0,1\}$ be an unknown $L$-hidden-layer ReLU network obtained from the PRF family by picking the key $s$ at random. Consider the lifted $(L+2)$-hidden-layer ReLU network $f_s^{\dv} : \R^d \to \R$ from \cref{eq:og-dv-lift}, given by $f_s^{\dv}(z) = \relu(f_s(N_1(z)) - N_2'(z))$, where $N_1$ and $N_2$ are from \cref{lem:n1,,lem:n2}, and $N_2'(z) = \sum_j N_2(z_j)$. Suppose there were an efficient learner $A$ capable of learning functions of the form $f_s^{\dv}$ using queries. By the DV lift (\cref{thm:og-dv}), $A$ yields an efficient predictor $B$ achieving small constant error w.r.t.\ the unknown $f_s$, contradicting \cref{lem:prf-learning-hardness}. We only need to verify that $A$'s query access to $f_s^{\dv}$ can be simulated by $B$. Indeed, suppose $A$ makes a query to $f_s^{\dv}$ at a point $z \in \R^d$. Then $B$ can make a query to $f_s$ at the point $\sgn(z)$ and return $\relu(f_s(\sgn(z)) - N_2'(z)) = f_s^{\dv}(z)$, as this was the key property satisfied by $f_s^{\dv}$. This completes the reduction and proves the theorem.
	\end{proof}
	
	\paragraph{Acknowledgments.} We would like to thank our anonymous reviewers for pointing out an issue in the first version of our proof. Part of this work was completed while the authors were visiting the Simons Institute for the Theory of Computing.
	
	\bibliography{refs}
	\bibliographystyle{alpha}
	
	\appendix
	
	\section{Barriers for constructing $N_3$}
	\label{app:barriers}
	
	We briefly discuss why one natural approach to constructing $N_3$ satisfying the ideal properties in \cref{eq:ideal} ultimately requires two hidden layers rather than one, unlike the construction we give in Sections~\ref{subsec:comp-lift} and \ref{subsec:truncation}.
	
	The most straightforward way to ensure that a function of $s_1,\ldots,s_d,t$ vanishes whenever there exists $j$ for which $s_j = 1$ would be to threshold on $\sum s_j$, e.g. by taking $\relu(1 - \sum_j s_j)$. While this function is a one-hidden-layer ReLU network, it is unclear how to modify it to satisfy the remaining desiderata in \eqref{eq:ideal} while preserving the fact that it has only one hidden layer. We note that \cite{daniely2021local} takes this approach of thresholding on $\sum_j s_j$ but uses two hidden layers.
	
	Here we informally argue that such an approach inherently requires an extra hidden layer. That is, we argue that no function $N:\R^2\to\R$ that takes as inputs $s\triangleq \sum_j s_j$ and $t$ and satisfies \eqref{eq:ideal} can be implemented as a one-hidden-layer network. Concretely, $N(s,t)$ must vanish whenever $s \ge 1$ or $t\in\mathbb{Z}\backslash \brc{0}$. Any function computed by a one-hidden-layer ReLU network of the form $(s, t) \mapsto \sum_{i}\relu(a_i s + b_i t - c_i)$, unless if it is affine linear, must in general be nowhere smooth (i.e.\ have a discontinuous gradient) along the entire line where a particular neuron of the network vanishes. In our example, these are the lines $\brc{(s,t): a_i s + b_i t = c_i }$. But this means that such a line cannot intersect the region $\brc{(s,t): s\ge 1}$, as otherwise it would be zero (hence smooth) on an infinite segment of the line. This can only happen if $b_i = 0$, i.e. none of the neurons of $N$ depend on $t$. Such a network clearly cannot satisfy \eqref{eq:ideal}.

	\section{Supporting lemmas for \cref{sec:lift}}
	\begin{lemma}\label{lem:altsum}
		For any $0\le S < m \le d$,
		\begin{equation}
			\sum^m_{i=S} (-1)^{m-i}\binom{d-i-2}{d-m-2}\binom{d-m-1}{i-S} = 0. \label{eq:combo}
		\end{equation}
	\end{lemma}
	
	\begin{proof}
		We will show that for any integers $j\ge \ell \ge 1$, 
		\begin{equation}
			\sum^\ell_{k = 0} (-1)^k \binom{j - k}{\ell - 1}\binom{\ell}{k} = 0. \label{eq:PIE}
		\end{equation} We would like to substitute $\ell = d - m - 1$ and $j = d - 2 - S$. Note that this is valid as we can assume without loss of generality that $d - m - 1 \ge 1$ (otherwise $\binom{d-m-1}{i-S} = 0$ on the right-hand side of \eqref{eq:combo}), and $j \ge \ell$ by our assumption that $S < m$. We conclude the identity
		\begin{equation}
			0 = \sum^{d - m  - 1}_{k = 0} (-1)^k \binom{d - 2 - S - k}{d - m - 2} \binom{d - m - 1}{k} = \sum^{d-m-1 +S}_{i=S} (-1)^{i-S} \binom{d - i - 2}{d - m - 2}\binom{d- m - 1}{i - S}, \label{eq:plugin}
		\end{equation}
		where the second step is by the change of variable $i = k +S$.
		If $d - m - 1 + S \ge m$, then note that all summands $m < i  \le d-m-1+S$ vanish because in that case $d-i-2 < d-m-2$ and so $\binom{d-i-2}{d-m-2} = 0$. If $d - m - 1 +S < m$, then note that all summands $d-m-1+S < i \le m$ vanish because in that case $d-m-1 < i-S$ and so $\binom{d-m-1}{i-S} = 0$. We conclude that \eqref{eq:plugin} is equal, up to a sign, to the left-hand side of \eqref{eq:combo}, so we'd be done.
		
		It remains to establish \eqref{eq:PIE}, which we do by following an argument due to \cite{stackexchange2019}. Observe that the left-hand side of \eqref{eq:PIE} is simply counting via inclusion-exclusion the number of subsets of $\brc{1,\ldots,j}$ of size $\ell - 1$ which contain $\brc{1,\ldots,\ell}$. Indeed, the $k = 0$ summand counts all subsets of size $\ell - 1$. The $k = 1$ summands subtract out the contribution, for every $1\le x\le \ell$, from the subsets of size $\ell - 1$ which contain $x$. The $k = 2$ summands add back the contribution, for every distinct $1\le x < y \le \ell$, from the subsets of size $\ell - 1$ which contain both of $x,y$, etc.
	\end{proof}
	
	\begin{lemma}\label{lem:doublesum}
		For any integers $m\ge 3$ and $a\in\brc{0,1,2}$,
		\begin{equation}
			\sum^m_{i=1}\sum^{m+1-i}_{j=1}(-1)^{m-i}\binom{d-i-j-1}{m-i-j+1}\binom{d-1}{i-1}\cdot j^a = \ind[a = 0]
		\end{equation}
		\begin{equation}
			\sum^m_{i=0}\sum^{m+1-i}_{j=1}(-1)^{m-i}\binom{d-i-j-1}{m-i-j+1}\binom{d-1}{i}\cdot j^a = 0
		\end{equation}
	\end{lemma}
	
	\begin{proof}
		By taking $\ell = i + j$, we can rewrite these sums as
		\begin{equation}
			S_{a,m} \triangleq \sum^{m+1}_{\ell = 2}\sum^{\ell-1}_{i=1} (-1)^{m-i}\binom{d-1-\ell}{m+1-\ell}\binom{d-1}{i-1}(\ell-i)^a \label{eq:rewrite_dubsum1}
		\end{equation}
		\begin{equation}
			T_{a,m} \triangleq \sum^{m+1}_{\ell=1}\sum^{\ell-1}_{i=0} (-1)^{m-i}\binom{d-1-\ell}{m+1-\ell}\binom{d-1}{i}(\ell-i)^a \label{eq:rewrite_dubsum2}
		\end{equation}
		We proceed by induction on $m$. The base cases follow from a direct calculation.
		By the change of variable $\ell' = \ell - 1$, we can rewrite $S_{a,m+1}$ as
		\begin{align}
			\MoveEqLeft -\sum^{m+1}_{\ell'=1}\sum^{\ell'}_{i=1}(-1)^{m-i}\binom{d-1-\ell'}{m+1-\ell'}\binom{d-1}{i-1} (\ell'+1-i)^a\\
			&= -\sum^{m+1}_{\ell'=1}\sum^{\ell'}_{i=1}(-1)^{m-i}\binom{d-1-\ell'}{m+1-\ell'}\binom{d-1}{i-1}(\ell'-i)^a \\  &\quad - \sum^{m+1}_{\ell'=1}\sum^{\ell'}_{i=1}(-1)^{m-i}\binom{d-1-\ell'}{m+1-\ell'}\binom{d-1}{i-1} \sum^{a-1}_{b=0} \binom{a}{b}(\ell'-i)^b \label{eq:intermed} \\ 
			\intertext{Note that the first term on the right-hand side differs from $S_{a,m}$ only in the summands given by $1\le i = \ell'\le m+1$, and those summands clearly vanish. We conclude that the first term on the right-hand side of \eqref{eq:intermed} is exactly $S_{a,m}$. For the second term on the right-hand side of \eqref{eq:intermed}, the part coming from any $0 < b \le a - 1$ is also zero, so we thus get}
			&= S_{a,m} - \sum^{m+1}_{\ell'=1}\sum^{\ell'}_{i=1}(-1)^{m-i}\binom{d-1-\ell'}{m+1-\ell'}\binom{d-1}{i-1} \\
			&= S_{a,m} - S_{0,m} -\sum^{m+1}_{\ell'=1}(-1)^{m-\ell'} \binom{d-1-\ell'}{m+1-\ell'}\binom{d-1}{\ell'-1} \\
			&= S_{a,m} -1 - \sum^{m+1}_{\ell'=1}(-1)^{m-\ell'}\binom{d-1-\ell'}{m+1-\ell'}\binom{d-1}{\ell'-1} \\
			&= S_{a,m} = \ind[a=0], \label{eq:altsum2}
		\end{align}
		where the penultimate step follows e.g.\ by applying the identity in \cite{stackexchange2017}. This completes the induction for $S_{a,m}$.
		
		For $T_{a,m}$, note that by the change of variable $i' = i + 1$,
		\begin{align}
			T_{a,m} &= -\sum^{m+1}_{\ell=1} \sum^\ell_{i'=1} (-1)^{m-i'} \binom{d-1-\ell}{m+1-\ell}\binom{d-1}{i'-1} (\ell - i'+ 1)^a \\
			&= -\sum^{m+1}_{\ell=2}\sum^{\ell-1}_{i'=1}(-1)^{m-i'}\binom{d-1-\ell}{m+1-\ell}\binom{d-1}{i'-1}(\ell-i'+1)^a - \sum^{m+1}_{\ell=1}(-1)^{m-\ell}\binom{d-1-\ell}{m+1-\ell}\binom{d-1}{\ell-1} \\
			&= -\sum^a_{b=0} \binom{a}{b}S_{b,m} + 1 = 0,
		\end{align}
		where in the second step we pulled out the summands corresponding to $i' = \ell$, in the third step we used \eqref{eq:altsum2}, and in the last step we used that for $m \ge 3$, $S_{b,m} = \ind[b \neq 0]$ for $0 \le b \le 2$.
	\end{proof}
	
	\section{SQ lower bound for the LWR functions}\label{app:lwr-sq}
	
	Here we prove an SQ lower bound for the $\lwr$ functions (\cref{thm:lwr-sq-lb}) using a general formulation in terms of pairwise independent function families. To our knowledge, this particular formulation has not appeared explicitly before in the literature, and was communicated to us by Bogdanov \cite{bogdanov2021personal}. A variant of this argument may be found in \cite[\S 7.7]{bogdanov2017pseudorandom}.
	
	\begin{definition}
		Let $\calC$ be a function family mapping $\calX$ to $\calY$, and let $D$ be a distribution on $\calX$. We call $\calC$ an $(1 - \eta)$-pairwise independent function family if with probability $1- \eta$ over the choice of $x, x'$ drawn independently from $D$, the distribution of $(f(x), f(x'))$ for $f$ drawn uniformly at random from $\calC$ is the product distribution $\unif(\calY) \otimes \unif(\calY)$.
	\end{definition}
	
	\begin{lemma}\label{lem:lwr-pairwise-ind}
		Fix security parameter $n$ and moduli $p, q$. The $\lwr_{n, p, q}$ function class $\calC_{\lwr} = \{f_w \mid w \in \Z_q^n\}$ is $(1 - \frac{2}{q^{n-1}})$-pairwise independent with respect to $\unif(\Z_q^n)$.
	\end{lemma}
	\begin{proof}
		This follows from the simple observation that whenever $x, x' \in \Z_q^n$ are linearly independent, the pair $(w \cdot x \bmod q, w \cdot x' \bmod q)$ for $w \sim \unif\{Z_q^n\}$ is distributed as $\unif(\Z_q) \otimes \unif(\Z_q)$. For such $x, x'$, $(f_w(x), f_w(x')) = (\frac{1}{p} \round{w \cdot x \bmod q}_p), \frac{1}{p} \round{w \cdot x' \bmod q}_p)$ for $f_w \sim \unif(\calC_{\lwr})$ is distributed as $\unif(\Z_p/p) \otimes \unif(\Z_p/p)$. The probability that $x, x' \sim \unif(\Z_q^n)$ are linearly dependent is at most \[ \pr[x = 0] + \pr[x \neq 0] \pr[x' \text{ is a multiple of } x] \leq \frac{1}{q^n} + \frac{q}{q^n} \leq \frac{2}{q^{n-1}}. \]
	\end{proof}
	
	We can now prove full SQ lower bounds for any $(1-\eta)$-pairwise independent function family as follows.
	
	\begin{lemma}\label{lem:pairwise-ind-var}
		Let $\calC$ mapping $\calX$ to $\calY$ be a $(1-\eta)$-pairwise independent function family w.r.t.\ a distribution $D$ on $\calX$. Let $\phi : \calX \times \calY \to [-1, 1]$ be any bounded query function. Then \[ \var_{f \sim \unif(\calC)} \ex_{x \sim D}\ [\phi(x, f(x))] \leq 2\eta. \]
	\end{lemma}
	\begin{proof}
		Denote $\ex_{x \sim D} [\phi(x, f(x))]$ by $\phi[f]$. By some algebraic manipulations (with all subscripts denoting independent draws),
		\begin{align}
			\var_{f \sim \unif(\calC)} \left[\phi[f]\right] &= \ex_{f} \left[ \phi[f]^2 \right] - \big(\ex_{f} \left[\phi[f]\right] \big)^2 \\
			&= \ex_{f} \left[ \phi[f] \phi[f] \right] - \ex_{f} \left[\phi[f]\right] \ex_{f'} \left[\phi[f']\right] \\
			&= \ex_{f, f'} \left[ \ex_{x} [\phi(x, f(x))] \ex_{x'} [\phi(x', f(x'))] - \ex_{x} [\phi(x, f(x))] \ex_{x'} [\phi(x', f'(x'))] \right] \\
			&= \ex_{x, x'} \ex_{f, f'} \left[ \phi(x, f(x)) \phi(x', f(x')) - \phi(x, f(x)) \phi(x', f'(x')) \right].
		\end{align} By $(1-\eta)$-pairwise independence of $\calC$, the inner expectation vanishes with probability $1- \eta$ over the choice of $x, x' \sim D$, and is at most 2 otherwise. This gives the claim.
	\end{proof}
	
	\begin{theorem}\label{thm:pairwise-ind-sq-lower}
		Let $\calC$ mapping $\calX$ to $\calY$ be a $(1-\eta)$-pairwise independent function family w.r.t.\ a distribution $D$ on $\calX$. For any $f \in \calC$, let $D_f$ denote the distribution of $(x, f(x))$ where $x \sim D$. Let $D_{\unif(\calC)}$ denote the distribution of $(x, y)$ where $x \sim D$ and $y = f(x)$ for $f \sim \unif(\calC)$ (this can be thought of as essentially $D \otimes \unif(\calY)$). Any SQ learner able to distinguish the labeled distribution $D_{f^*}$ for an unknown $f^* \in \calC$ from the randomly labeled distribution $D_{\unif(\calC)}$ using bounded queries of tolerance $\tau$ requires at least $\frac{ \tau^2}{2\eta}$ such queries.
	\end{theorem}
	\begin{proof}
		Let $\phi : \calX \times \calY \to [-1, 1]$ be any query made by the learner. For any $f \in \calC$, let $\phi[f]$ denote $\ex_{x \sim D}[\phi(x, f(x))] = \ex_{(x, y) \sim D_f} [\phi(x, y)]$. Consider the adversarial strategy where the SQ oracle responds to this query with $\overline{\phi} = \ex_{f \sim \unif(\calC)} \phi[f] = \ex_{(x, y) \sim D_{\unif(\calC)}} [\phi(x, y)]$. By Chebyshev's inequality and \cref{lem:pairwise-ind-var}, \[ \pr_{f \sim \calC} \left[ \big|\phi[f] - \overline{\phi}\big| > \tau \right] \leq \frac{\var_{f \sim \unif(\calC)}\big[\phi[f]\big]}{\tau^2} \leq \frac{2\eta}{\tau^2}.  \] So each such query only allows the learner to rule out at most a $\frac{2\eta}{\tau^2}$ fraction of $\calC$. Thus to distinguish $D_{f^*}$ from $D_{\unif(\calC)}$, the learner requires at least $\frac{ \tau^2}{2\eta}$ queries.
	\end{proof}
	
	\cref{thm:lwr-sq-lb} now follows easily as a corollary.
	
	\begin{proof}[Proof of \cref{thm:lwr-sq-lb}]
		It is not hard to see that learning $\calC_{\lwr}$ up to squared loss $1/16$ certainly suffices to solve the distinguishing problem in \cref{thm:pairwise-ind-sq-lower}. The claim now follows by \cref{lem:lwr-pairwise-ind}.
	\end{proof}
	
	\begin{remark}
	    We remark that the argument in this section, specialized to the $q=2$ case, recovers the traditional SQ lower bound for parities (\cref{thm:parity-lb}) without appealing to any notion of statistical dimension.
	\end{remark}
	
\end{document}